\documentclass{article}
\pdfoutput=1

\usepackage[textsize=tiny]{todonotes}
\usepackage{microtype}
\usepackage{graphicx}
\usepackage{subfigure}
\usepackage{booktabs} 
\usepackage{amsmath}
\usepackage{amssymb}
\usepackage{dsfont}
\usepackage{xspace} 
\usepackage{amsthm}  
\usepackage{thmtools}
\usepackage{thm-restate}
\usepackage{etoolbox}
\usepackage{bm}
\usepackage{amsfonts}
\usepackage{caption}
\usepackage{hyperref}

\newcommand{\overbar}[1]{\mkern 1.5mu\overline{\mkern-1.5mu#1\mkern-1.5mu}\mkern 1.5mu}

\declaretheorem[name=Proposition,numberwithin=section]{prop}

\declaretheorem[name=Lemma,numberwithin=section]{lemma}

\DeclareRobustCommand{\eg}{e.g.,\@\xspace}                         
\DeclareRobustCommand{\ie}{i.e.,\@\xspace}                         
\DeclareRobustCommand{\wrt}{w.r.t.\@\xspace}                       

\newcommand{\mathbr}[1]{\bm{\mathbf{#1}}}
\DeclareRobustCommand{\quotes}[1]{``#1''}

\AtBeginEnvironment{proof}{\small}
\allowdisplaybreaks[4]

\DeclareMathOperator*{\argmax}{arg\,max}

\newcommand{\ev}{\mathop{\mathds{E}}}

\usepackage{hyperref}

\usepackage[accepted]{icml2018}

\icmltitlerunning{Configurable Markov Decision Processes}

\begin{document}

\setlength{\abovedisplayskip}{2pt}
\setlength{\belowdisplayskip}{2pt}

\twocolumn[
\icmltitle{Configurable Markov Decision Processes}



\icmlsetsymbol{equal}{*}

\begin{icmlauthorlist}
\icmlauthor{Alberto Maria Metelli}{polimi,equal}
\icmlauthor{Mirco Mutti}{polimi,equal}
\icmlauthor{Marcello Restelli}{polimi}
\end{icmlauthorlist}

\icmlaffiliation{polimi}{Politecnico di Milano, 32, Piazza Leonardo da Vinci, Milan,
Italy}

\icmlcorrespondingauthor{Alberto Maria Metelli}{\href{mailto:albertomaria.metelli@polimi.it}{\texttt{albertomaria.metelli@polimi.it}}}

\icmlkeywords{Machine Learning, Reinforcement Learning, Configurable Markov Decision Processes, Safe Learning, ICML}

\vskip 0.3in
]



\printAffiliationsAndNotice{\icmlEqualContribution} 

\begin{abstract}
In many real-world problems, there is the possibility to configure, to a limited extent, some environmental parameters to improve the performance of a learning agent. In this paper, we propose a novel framework, Configurable Markov Decision Processes (Conf-MDPs), to model this new type of interaction with the environment. Furthermore, we provide a new learning algorithm, Safe Policy-Model Iteration (SPMI), to jointly and adaptively optimize the policy and the environment configuration. After having introduced our approach and derived some theoretical results, we present the experimental evaluation in two explicative problems to show the benefits of the environment configurability on the performance of the learned policy.
\end{abstract}

\section{Introduction}
\label{sec:introduction}
Markov Decision Processes (MDPs)~\cite{puterman2014markov} are a popular formalism to model sequential decision-making problems. Solving an MDP means to find a \emph{policy}, \ie a prescription of actions, that maximizes a given utility function.
Typically, the environment
dynamics is assumed to be fixed, unknown and out of the control of the agent. Several exceptions to this scenario can be found in the literature, especially in the context of Markov Decision Processes with imprecise probabilities (MDPIPs)~\cite{satia1973markovian, white1994markov, bueno2017modeling} and non-stationary environments~\cite{bowerman1974nonstationary, hopp1987new}. In the former case, the transition kernel is known under uncertainty. Therefore, it
cannot be specified using a conditional probability distribution, but it must be defined through a set of probability distributions~\cite{delgado2009representing}. In this context, Bounded-parameter Markov Decision Processes (BMDPs) consider a special case in which upper and lower bounds on transition probabilities are specified~\cite{givan1997boundedmdp, ni2008boundedpomdp}.
A common approach is to solve a minimax problem to find a robust policy maximizing the expected return under the worst possible transition model~\cite{osogami2015robustpomdp}.
In non-stationary environments, the transition probabilities (possibly also the reward function) change over time~\cite{bowerman1974nonstationary}. Several works tackle the problem of defining an optimality criterion~\cite{hopp1987new} and finding optimal
policies in non-stationary environments~\cite{garcia2000solving, cheevaprawatdomrong2007solution, ghate2013linear}.

Although the environment is no longer fixed, both Markov decision processes with imprecise probabilities and non-stationary Markov decision processes do not admit the possibility to dynamically alter the environmental parameters. However, there exist several real-world scenarios in
which the environment is \emph{partially controllable} and, therefore, it might be beneficial to configure some of its features in order to select the most convenient MDP to solve. For instance, a human car driver has at her/his disposal a number
of possible vehicle configurations she/he can act on (\eg seasonal tires, stability and vehicle attitude, engine model, automatic speed control, parking aid system) to improve the driving style or quicken the process of learning a good driving policy. Another example is the interaction between a student and an automatic teaching system: the teaching model can be tailored to improve the student's learning experience (\eg increasing or decreasing the difficulties of the questions or the speed at which the concepts are presented). It is worth noting that the active entity in the configuration process might be the agent itself or an external supervisor guiding the learning process. In the latter case, for instance, a supervisor can dynamically adapt where to place the products in a supermarket in order to maximize the customer (agent) satisfaction. Similarly, the design of a street network could be 
configured, by changing the semaphore transition times or the direction of motion, to reduce the drivers' journey time. In a more abstract sense, the possibility to act on the environmental parameters can have essentially two benefits: i) it allows improving the agent performance; ii) it may allow to speed up the learning process. This second instance has been previously addressed in \cite{ciosek2017offer, florensa2017reversecurriculum}, where the transition model and the initial state distribution are altered in order to reach a faster convergence to the optimal policy. However, in both the cases the environment modification is only simulated, while the underlying environment dynamic remains unchanged.

In this paper, we propose a framework
to model a \emph{Configurable Markov Decision Process} (Conf-MDP), \ie an MDP in which the environment can be configured to a limited extent. In principle, any of the Conf-MDP's parameters can be tuned, but we restrict our attention to the transition model and we focus to the problem of identifying the environment that allows achieving the highest performance possible. At an intuitive level, there exists a tight connection between environment and policy: variations of the environment induce modifications of the optimal policy. Furthermore, even for the same task, in presence of agents having access to different policy spaces, the optimal environment might be different.\footnote{In general, a modification of the environment (\eg changing the configuration of a car) is more expensive and more constrained \wrt to a modification of the policy.} The spirit of this work is to investigate and exercise the tight connection between policy and model, pursuing the goal of improving the final policy performance. After having introduced the definition of Conf-MDP, we propose a method to \emph{jointly} and \emph{adaptively} optimize the policy and the transition model, named \emph{Safe Policy-Model Iteration} (SPMI). The algorithm adopts a safe learning approach~\cite{pirotta2013safe} based on the maximization of a lower bound on the guaranteed performance improvement, yielding a sequence of model-policy pairs with monotonically increasing performance. The safe learning perspective makes our approach suitable for critical applications where performance degradation during learning is not allowed (\eg industrial scenarios where extensive exploration of the policy space might damage the machinery). In the standard Reinforcement Learning (RL) framework~\cite{sutton1998reinforcement}, the usage of a lower bound to guide the choice of the policy has been first introduced by Conservative Policy Iteration (CPI)~\cite{kakade2002approximately}, improved by Safe Policy Iteration (SPI)~\cite{pirotta2013safe} and subsequently exploited by~\cite{ghavamzadeh2016safe, abbasi2016fast, papini2017adaptive}. These methods revealed their potential thanks to the preference towards small policy updates, preventing from moving in a single step too far away from the current policy and avoiding premature convergence to suboptimal policies. A similar rationale is at the basis of Relative Entropy Policy Search (REPS)~\cite{peters2010relative}, and, more recently, Trust Region Policy Optimization (TRPO)~\cite{schulman2015trust} and Proximal Policy Optimization (PPO)~\cite{schulman2017proximal}. In order to introduce our framework and highlight its benefits, we limit our analysis to the scenario in which the model space (and the policy space) is known. However, when the model space is unknown, we could resort to a sample-based version of SPMI, which could be derived by adapting that of SPI~\cite{pirotta2013safe}.

We start in Section~\ref{sec:preliminaries} by recalling some basic notions about MDPs and providing the definition of Conf-MDP. In Section~\ref{sec:perf_impr} we derive the performance improvement bound and we outline the main features of SPMI (Section~\ref{sec:spmi}) along with some theoretical results (Section~\ref{sec:theo}).\footnote{The proofs of all the lemmas and theorems can be found in Appendix~\ref{apx:proofs}.} Then, we present the experimental evaluation (Section~\ref{sec:experimental}) in two explicative domains, representing simple abstractions of the main application of Conf-MDPs, with the purpose of showing how configuring the transition model can be beneficial for the final policy performance. 

\section{Preliminaries}
\label{sec:preliminaries}
A discrete-time Markov Decision Process (MDP)~\cite{puterman2014markov} is defined as
$\mathcal{M}=(\mathcal{S},\mathcal{A},P,R,\gamma,\mu)$ where
$\mathcal{S}$ is the state space, $\mathcal{A}$ is the action space,
$P(s'|s,a)$ is a Markovian transition model that defines the
conditional distribution of the next state $s'$ given the current state $s$ and
the current action $a$, $\gamma\in{(0,1)}$ is the discount factor, $R(s,a)\in [ 0 , 1 ]$ is
the reward for performing action $a$ in state $s$ and $\mu$ is the distribution of the initial state. A policy $\pi(a|s)$ defines the probability distribution of an action $a$ given the current state $s$. Given a model-policy pair $(P,\pi)$ we
indicate with $P^{\pi}$ the state kernel function defined as $P^{\pi}(s'|s) = \int_{\mathcal{A}} \pi(a|s)P(s'|s,a) \mathrm{d}a$. We now formalize the Configurable Markov Decision Process (Conf-MDP).

\begin{restatable}[]{defi}{}
	A \textup{Configurable Markov Decision Process} is a tuple $\mathcal{CM} = (\mathcal{S}, \mathcal{A}, R, \gamma, \mu, \mathcal{P}, \Pi)$ where $(\mathcal{S}, \mathcal{A}, R, \gamma, \mu)$ is an MDP without the transition model and $\mathcal{P}$ and $\Pi$ are the model and policy spaces.
\end{restatable}

More specifically, $\Pi$ is the set of policies the agent has access to and $\mathcal{P}$ is the set of available environment configurations (transition models). The performance of a model-policy pair $(P,\pi) \in \mathcal{P} \times \Pi$ is evaluated through the \emph{expected return}, \ie the expected discounted sum of the rewards collected along a trajectory:
\begin{equation}
	J^{P,\pi}_\mu = \frac{1}{1-\gamma} \int_{\mathcal{S}}  d_{\mu}^{P,\pi}(s) \int_{\mathcal{A}} \pi(a|s) R(s,a) \mathrm{d}a \mathrm{d}s,
\end{equation}
where $d_{\mu}^{P,\pi}$ is the $\gamma$-discounted state distribution~\cite{sutton2000policy}, defined recursively as:
\begin{equation} \label{eq:d_mu}
	d_{\mu}^{P,\pi}(s) = (1-\gamma) \mu(s) + \gamma \int_{\mathcal{S}} d_{\mu}^{P,\pi}(s')P^{\pi}(s'|s) \mathrm{d}s'.
\end{equation}
We can also define the $\gamma$-discounted state-action distribution as $\delta_{\mu}^{P,\pi}(s,a) = \pi(a|s)d_{\mu}^{P,\pi}(s)$.
While solving an MDP consists in finding a policy $\pi^*$ that maximizes $J^{P,\pi}_\mu$ under the given fixed environment $P$, solving a Conf-MDP consists in finding a model-policy pair $(P^*,\pi^*)$ such that $P^*,\pi^* = \argmax_{P \in \mathcal{P}, \pi \in \Pi} J^{\pi,P}_\mu$.
For control purposes, the state-action value function, or \emph{Q-function}, is introduced as the expected return starting from a state $s$ and performing action $a$:
\begin{equation} \label{eq:Q}
	Q^{P,\pi}(s,a) =  R(s,a) + \gamma \int_{\mathcal{S}} P(s'|s,a) V^{P,\pi}(s') \mathrm{d}s'.
\end{equation}
For learning the transition model we introduce the state-action-next-state value function or \emph{U-function}, defined as the expected return starting from the state $s$, performing action $a$ and landing to state $s'$:
\begin{equation} \label{eq:U}
	U^{P,\pi}(s,a,s') = R(s,a) + \gamma V^{P,\pi}(s'),
\end{equation}
where $V^{P,\pi}$ is the state value function or \emph{V-function}. These three functions are tightly connected due to the trivial relations: $V^{P,\pi}(s) = \int_{\mathcal{A}} \pi(a|s) Q^{P,\pi}(s,a)\mathrm{d}a$ and $Q^{P,\pi}(s,a) = \int_{\mathcal{S}} P(s'|s,a) U^{P,\pi}(s,a,s')\mathrm{d}s'$.
Furthermore, we define the \emph{policy advantage function} $A^{P, \pi} (s,a) = Q^{P, \pi}(s,a) - V^{P, \pi}(s)$ that quantifies how much an action is better than the others and the \emph{model advantage function} $A^{P, \pi} (s,a,s') = U^{P, \pi}(s,a,s') - Q^{P, \pi}(s,a)$ that quantifies how much the next state is better than the other ones.
In order to evaluate the \emph{one-step improvement} in performance attained by a new policy $\pi'$ or model $P'$ when the current policy is $\pi$ and the current model is $P$, we introduce the \emph{relative advantage functions}~\cite{kakade2002approximately}:
\begin{gather*}
	A_{P, \pi}^{P, \pi'}(s) = \int_{\mathcal{A}} \pi'(a|s) A^{P,\pi}(s,a) \mathrm{d}a,\\
    A_{P, \pi}^{P', \pi}(s,a) = \int_{\mathcal{S}} P'(s'|s,a) A^{P,\pi}(s,a,s') \mathrm{d}s',
\end{gather*}
and the corresponding expected values under the $\gamma$-discounted distributions: $
\mathds{A}_{P,\pi,\mu}^{P,\pi'} = \int_{\mathcal{S}} d_{\mu}^{P,\pi}(s) A_{P,\pi}^{P,\pi'}(s) \mathrm{d}s$ and $\mathds{A}_{P,\pi,\mu}^{P',\pi} = \int_{\mathcal{S}} \int_{\mathcal{A}} \delta_{\mu}^{P,\pi}(s,a) A_{P,\pi}^{P',\pi}(s,a) \mathrm{d}s\mathrm{d}a$.

\section{Performance Improvement}
\label{sec:perf_impr}
The goal of this section is to provide a lower bound to
the performance improvement obtained by moving from a model-policy pair $(P,\pi)$ to another pair $(P',\pi')$.

\subsection{Bound on the $\gamma$-discounted distribution}
We start providing a bound for 
the difference of $\gamma$-discounted distributions under different model-policy pairs.

\begin{restatable}[]{prop}{distributionsBoundCoupled}
	Let $(P,\pi)$ and $(P',\pi')$ be two model-policy pairs,
    the $\ell^1$-norm of the difference between the $\gamma$-discounted state distributions can be upper bounded as:
    \begin{equation*}
    	\Big\| d_{\mu}^{P',\pi'} - d_{\mu}^{P,\pi} \Big\|_1 \le  \frac{\gamma}{1-\gamma} D_{\ev}^{{P'}^{\pi'},{P}^{\pi}},
    \end{equation*}
    where $D_{\ev}^{{P'}^{\pi'},{P}^{\pi}} = \ev_{s\sim d_{\mu}^{P,\pi}} \big\| {P'}^{\pi'}(\cdot|s) - P^{\pi}(\cdot|s) \big\|_1$.
\end{restatable} \label{thr:distributionsBoundCoupled}
This proposition provides a way to upper bound the difference of the $\gamma$-discounted state distributions
in terms of the state kernel dissimilarity.\footnote{More formally, $D_{\ev}^{{P'}^{\pi'},{P}^{\pi}}$ is just a \emph{premetric}~\cite{deza2009encyclopedia} and not a metric (see Appendix~\ref{apx:diss} for details).} The state kernel
couples the effects of the policy and the transition model, but it is convenient to keep their
contribution separated, getting the following looser bound.

\begin{restatable}[]{coroll}{distributionsBound}
	Let $(P,\pi)$ and $(P',\pi')$ be two model-policy pairs,
    the $\ell^1$-norm of the difference between the $\gamma$-discounted state distributions can be upper bounded as:
    \begin{equation*}
    	\Big\| d_{\mu}^{P',\pi'} - d_{\mu}^{P,\pi} \Big\|_1 \le  \frac{\gamma}{1-\gamma}  \Big( D_{\ev}^{\pi', \pi} + D_{\ev}^{P', P} \Big),
    \end{equation*}
     where $D_{\ev}^{\pi', \pi} = \ev_{s\sim d_{\mu}^{P,\pi}} \big\| \pi'(\cdot|s) - \pi(\cdot|s) \big\|_1$ and $D_{\ev}^{P', P} = \ev_{(s,a)\sim \delta_{\mu}^{P,\pi}} \big\| P'(\cdot|s,a) - P(\cdot|s,a) \big\|_1$.
\end{restatable} \label{thr:distributionsBound}
It is worth noting that when $P=P'$ the bound resembles Corollary 3.2 in~\cite{pirotta2013safe}, but it is tighter as:
\begin{equation*}
\ev_{s\sim d_{\mu}^{P,\pi}} \big\| \pi'(\cdot|s) - \pi(\cdot|s) \big\|_1 \le \sup_{s\in\mathcal{S}} \big\| \pi'(\cdot|s) - \pi(\cdot|s) \big\|_1,
\end{equation*}
in particular the bound of~\cite{pirotta2013safe} might yield a large bound value in case there exist states in which the policies are very
different even if those states are rarely visited according to $d_{\mu}^{P,\pi}$. In the context of policy learning, a lower bound employing the same dissimilarity index $D_{\ev}^{\pi', \pi}$ in the penalization term has been previously proposed in~\cite{achiam2017cpo}.

\subsection{Bound on the Performance Improvement}
In this section, we exploit the previous results to obtain a lower bound on the performance improvement
as an effect of the policy and model updates. We start introducing the \emph{coupled relative advantage function}:
\begin{equation*}
		A_{P,\pi}^{P',\pi'}(s) = \int_{\mathcal{S}} \int_{\mathcal{A}}  \pi'(a|s)P'(s'|s,a) \tilde{A}^{P,\pi}(s,a,s')\mathrm{d}s'\mathrm{d}a,
\end{equation*}
where $\tilde{A}^{P,\pi}(s,a,s') = U^{P,\pi}(s,a,s')  - V^{P,\pi}(s)$. $A_{P,\pi}^{P',\pi'}$ represents the one-step improvement attained by the new model-policy pair $(P',\pi')$ over
the current one $(P,\pi)$, \ie the local gain in performance yielded by selecting an action with $\pi'$ and the next state with $P'$. The corresponding expectation under the $\gamma$-discounted distribution is given by: $\mathds{A}_{P,\pi,\mu}^{P',\pi'} = \int_{\mathcal{S}} d_{\mu}^{P,\pi}(s) A_{P,\pi}^{P',\pi'}(s) \mathrm{d}s$.
Now, we have all the elements to express the performance improvement in terms of the relative advantage functions and the $\gamma$-discounted distributions.
\begin{restatable}[]{thr}{perfImprovement}
\label{thr:perfImprovement}
	The performance improvement of model-policy pair $(P',\pi')$ over $(P,\pi)$ is given by:
    \begin{equation*}
    	J^{P',\pi'}_\mu - J^{P,\pi}_\mu = \frac{1}{1-\gamma} \int_{\mathcal{S}} d_\mu^{P',\pi'}(s) A_{P,\pi}^{P',\pi'}(s) \mathrm{d}s.
    \end{equation*}
\end{restatable} 
This theorem is the natural extension of the result proposed by~\citet{kakade2002approximately}, but, unfortunately, it cannot be directly exploited in an algorithm
as the dependence of $d_\mu^{P',\pi'}$ on the candidate model-policy pair $(P',\pi')$ is highly nonlinear and difficult to treat. We aim to obtain, from this result, a lower bound on $J^{P',\pi'}_\mu - J^{P,\pi}_\mu$ that can be efficiently computed using information on the current pair $(P,\pi)$.

\begin{restatable}[Coupled Bound]{thr}{boundCoupled}
	The performance improvement of model-policy pair $(P',\pi')$ over $(P,\pi)$ can be lower bounded as:
    \begin{equation*}
		\underbrace{J^{{P}',\pi'}_{\mu} - J^{{P},\pi}_{\mu}}_{\substack{\text{performance}\\\text{improvement}}} \ge \underbrace{\frac{\mathds{A}_{{P},\pi,\mu}^{{P}',\pi'}}{1-\gamma}}_{\text{advantage}} - \underbrace{\frac{\gamma \Delta A^{{P'},\pi'}_{P,\pi}D_{\ev}^{{P'}^{\pi'},{P}^{\pi}}}{2(1-\gamma)^2}  }_{\text{dissimilarity penalization}},
	\end{equation*}
	where $\Delta A^{P',\pi'}_{P,\pi} = \sup_{s,s'\in\mathcal{S}} \big| A^{P',\pi'}_{P,\pi}(s') - A^{P',\pi'}_{P,\pi}(s) \big|$.
\end{restatable}
The bound is composed of two terms, like in~\cite{kakade2002approximately, pirotta2013safe}: the first term, \emph{advantage}, represents how much gain in performance can be locally obtained by moving from $(P,\pi)$ to $(P',\pi')$, whereas the second term, \emph{dissimilarity penalization}, discourages
updates towards model-policy pairs that are too far away. 

The \emph{coupled bound}, however, is not suitable to be used in an algorithm as it does not separate the
contribution of the policy and that of the model. In practice, an agent cannot directly update the kernel function $P^{\pi}$ since the environment model can only partially be controlled, whereas, in many cases, we can assume
a full control on the policy. For this reason, it is convenient to derive a bound in which the policy and model effects are decoupled.

\begin{restatable}[Decoupled Bound]{thr}{boundUncoupled}
\label{thr:boundUncoupled}
	The performance improvement of model-policy pair $(P',\pi')$ over $(P,\pi)$ can be lower bounded as:
	\begingroup
\addtolength{\jot}{-2em}
	\begin{align*}
		\underbrace{J^{{P}',\pi'}_{\mu}- J^{{P},\pi}_{\mu}}_{\substack{\text{performance}\\ \text{improvement}}} & \ge B(P',\pi') = \\ 
		& = \underbrace{\frac{\mathds{A}_{P,\pi,\mu}^{P',\pi} +  \mathds{A}_{P,\pi,\mu}^{P,\pi'}}{1-\gamma}}_{\text{advantage}} - \underbrace{\frac{\gamma \Delta Q^{{P},\pi} D  }{2 (1-\gamma)^2} }_{\substack{\text{dissimilarity}\\ \text{penalization}}},
	\end{align*}
	\endgroup
	where $D$ is a dissimilarity term defined as:
	\begin{equation*}
		D = D^{\pi',\pi}_{\ev} \big( D^{\pi',\pi}_{\infty} +  D^{P',P}_{\infty} \big)  + D^{P',P}_{\ev}  \big( D^{\pi',\pi}_{\infty} +  \gamma D^{P',P}_{\infty} \big),
	\end{equation*}
	 $D^{\pi',\pi}_{\infty}= \sup_{s\in\mathcal{S}} \big\| \pi'(\cdot|s) - \pi(\cdot|s) \big\|_1$, $D^{P',P}_{\infty}= \sup_{s\in\mathcal{S},a\in\mathcal{A}} \big\| P'(\cdot|s,a) - P(\cdot|s,a) \big\|_1$ and $\Delta Q^{P,\pi} = \sup_{s,s'\in\mathcal{S}, a,a'\in \mathcal{A}} \big| Q^{P,\pi}(s',a') - Q^{P,\pi}(s,a) \big|$.
\end{restatable}

\section{Safe Policy Model Iteration}
\label{sec:spmi}
To deal with the learning problem in the Conf-MDP framework we could, in principle, learn the optimal policy by using a classical RL algorithm and adapt it to learn the optimal model, sequentially or in parallel. Alternatively, we could resort to general-purpose global optimization tools, like CEM~\cite{rubinstein1999cross} or genetic algorithms~\citep{holland1989genetic}, using as objective function the performance of the policy learned by
a standard RL algorithm. Nonetheless, they may not correspond to the preferable, nor the safest, choices in this context as there exists an inherent connection between policy and model we could not overlook during the learning process. Indeed, a policy learned by interacting with a sub-optimal model could result in poor performance paired with a different, optimal model. At the same time, a policy far from the optimum could mislead the search for the optimal model. 
The goal of this section is to present an approach, \emph{Safe Policy-Model Iteration} (SPMI), inspired by~\citep{pirotta2013safe}, capable of learning the policy and the model simultaneously,\footnote{In the context of Conf-MDPs we believe that knowing the model of the configurable part of the environment is a reasonable requirement.} possibly taking advantage of the inter-connection mentioned above. 

\subsection{The Algorithm}
\begin{table*}[ht]
	\caption{The four possible optimal $(\alpha,\beta)$ pairs, the optimal pair is the one yielding the maximum bound value (all values are clipped in $[0,1]$). The corresponding guaranteed performance improvements can be found in Appendix~\ref{apx:proofs}.}
	\label{tab:coefficients}
	\vskip 0.15in
	\begin{center}
	\begin{small}
	\begin{tabular}{cccc}
		\toprule
		$\beta^*=0$ & $\alpha^*=0$ & $\beta^*=1$ & $\alpha^*=1$ \\
		\midrule
		$\alpha^*_0 = \frac{(1-\gamma) \mathds{A}_{P,\pi,\mu}^{P,\overline{\pi}}}{\gamma \Delta Q^{P,\pi} D_{\infty}^{\overline{\pi},\pi}D_{\mathds{E}}^{\overline{\pi},\pi}}$ & $\beta^*_0=\frac{(1-\gamma) \mathds{A}_{P,\pi,\mu}^{\overline{P},\pi}}{\gamma^2 \Delta Q^{P,\pi} D_{\infty}^{\overline{P},P}D_{\mathds{E}}^{\overline{P},P}}$ & $\alpha^*_1 = \alpha^*_0 - \frac{1}{2} \bigg( \frac{D_{\mathds{E}}^{\overline{P},P}}{D_{\mathds{E}}^{\overline{\pi},\pi}} + \frac{D_{\infty}^{\overline{P},P}}{D_{\infty}^{\overline{\pi},\pi}} \bigg)$ & $\beta^*_1 = \beta^*_0 -  \frac{1}{2\gamma} \bigg( \frac{D_{\mathds{E}}^{\overline{\pi},\pi}}{D_{\mathds{E}}^{\overline{P},P}} + \frac{D_{\infty}^{\overline{\pi},\pi}}{D_{\infty}^{\overline{P},P}} \bigg) $\\
		\bottomrule
	\end{tabular}
	\end{small}
	\end{center}
	\vskip -0.1in
\end{table*}

Following the approach proposed in~\citep{pirotta2013safe}, we define the policy and model improvement update rules:
\begin{equation*}
	\pi' = \alpha \overline{\pi} + (1-\alpha) \pi, \;\;  P'= \beta \overline{P} + (1-\beta) P,
\end{equation*}
where $\alpha,\beta \in [0,1]$, $\overline{\pi}\in\Pi$ and $\overline{P}\in\mathcal{P}$ are the target policy and the target model respectively. 
Extending the rationale of~\citep{pirotta2013safe} to our context, we aim to determine the values of $\alpha$ and $\beta$ which jointly maximize the \emph{decoupled bound} (Theorem~\ref{thr:boundUncoupled}). In the following we will abbreviate $B(P',\pi')$ with $B(\alpha,\beta)$.
\begin{restatable}[]{thr}{optimalUpdate} 
For any $\overline{\pi}\in\Pi$ and $\overline{P}\in\mathcal{P}$, the decoupled bound is optimized for:
\begin{equation*}
	\alpha^*, \beta^* = \argmax_{\alpha, \beta} \{B(\alpha,\beta): (\alpha,\beta)\in\mathcal{V} \},
\end{equation*}
where $\mathcal{V} = \{(\alpha^*_{0},0),(\alpha^*_{1},1),(0, \beta^*_{0}),(1, \beta^*_{1}) \}$ and the values of $\alpha^*_{0}$, $\alpha^*_{1}$, $\beta^*_{0}$ and $\beta^*_{1}$ are reported in Table~\ref{tab:coefficients}.
\end{restatable}
The theorem expresses the fact that the optimal $(\alpha,\beta)$ pair lies on the boundary of $[0,1]\times [0,1]$, \ie either one between policy and model is moved and the other is kept unchanged or one is moved and the other is set to target.

Algorithm~\ref{SPMI} reports the basic structure of SPMI. The algorithm stops when both the expected relative advantages fall below a threshold $\epsilon$. The procedures \emph{PolicyChooser} and \emph{ModelChooser} are designated for selecting the target policy and model (see Section~\ref{sec:targetChoice}).

\begin{algorithm}[t]
\caption{Safe Policy Model Iteration} \label{SPMI}
\begin{algorithmic}
\small
\STATE initialize $\pi_0,P_0$.
\FOR{$i = 0,1,2,... $ until $\epsilon$-convergence}
\STATE  $\overline{\pi}_i = \textit{PolicyChooser}(\pi_i)$
\STATE  $\overline{P}_i = \textit{ModelChooser}(P_i)$ 
\STATE $\mathcal{V}_i = \{(\alpha^*_{0,i},0),(\alpha^*_{1,i},1),(0, \beta^*_{0,i}),(1, \beta^*_{1,i}) \}$
\STATE $\alpha^*_i, \beta^*_i = \argmax_{\alpha, \beta} \{B(\alpha,\beta): (\alpha,\beta)\in\mathcal{V}_i \}$
\STATE $\pi_{i+1} = \alpha^*_i \overline{\pi}_i + (1-\alpha^*_i) \pi_i$
\STATE $P_{i+1}= \beta^*_i \overline{P}_i + (1-\beta^*_i) P_i$ 
\ENDFOR
\end{algorithmic}
\end{algorithm}

\subsection{Policy and Model Spaces}
The selection of the target policy and model is a rather crucial component of the algorithm since the quality of the updates largely depends on it.
To effectively adopt a target selection strategy we have to know which are the degrees of freedom on the policy and model spaces.
Focusing on the model space first, it is easy to discriminate two macro-classes in real-world scenarios. In some cases, there are almost no constraints on the direction in which to update the model. In other cases, only a limited model portion, typically a set of parameters inducing transition probabilities, can be accessed.
While we can naturally design the first scenario as an \emph{unconstrained} model space, to represent the second scenario we limit the model space to the convex hull $\mathrm{co}(\bm{P})$, where $\bm{P}$ is a set of extreme (or vertex) models. Since only the convex combination coefficients can be controlled, we refer to the latter as a \emph{parametric} model space.
It is noteworthy that we can symmetrically extend the dichotomy to the policy space, although the need to limit the agent on the direction of policy updates is less significant in our perspective.

\subsection{Target Choice}
\label{sec:targetChoice}
To deal with unconstrained spaces, it is quite natural to adopt the target selection strategy presented in~\citep{pirotta2013safe}, by introducing the concept of greedy model as $P^+(\cdot|s,a) \in \argmax_{s'\in\mathcal{S}} U^{P,\pi}(s,a,s')$, \ie the model that maximizes the relative advantage in each state-action pair. At each step, the greedy policy and model \wrt the $Q^{P,\pi}$ and $U^{P,\pi}$ are selected as targets.
When we are not free to choose the greedy model, like in the parametric setting, we select 
the vertex model that maximizes the expected relative advantage (\emph{greedy choice}).
The greedy strategy is based on local information and is not guaranteed to provide a model-policy pair maximizing the bound. However, testing all the model-policy pairs is highly inefficient in the presence of large model-policy spaces. A reasonable compromise is to select, as a target, the model that yields the maximum bound value between the greedy target $\overline{P}_i \in \argmax_{P \in \bm{P}} \mathds{A}_{P_i,\pi,\mu}^{P,\pi}$ and the previous target $\overline{P}_{i-1}$ (the same procedure can be employed for the policy). This procedure, named \emph{persistent choice}, effectively avoids the oscillating behavior, common with the greedy choice.

\section{Theoretical Analysis}
\label{sec:theo}
In this section, we outline some relevant theoretical results related to SPMI. We start by analyzing the scenario in which the model/policy space is parametric, \ie is limited to the convex hull of a set of vertex models/policies, and then we provide some rationales for the target choices adopted. In most of the section, we restrict our attention to the transition model, as for the policy all results apply symmetrically. 

\subsection{Parametric Model Space}
We consider the setting in which the transition model space is limited to the convex hull of a finite set of vertex models (\eg a set of deterministic models): $\mathcal{P} = \mathrm{co}(\bm{P})$, where $\bm{P} = \{P_1, P_2, ..., P_M\}$.
Each model in $\mathrm{co}(\bm{P})$ is defined by means of a  
coefficient vector $\mathbr{\omega}$ belonging to the $M$-dimensional fundamental simplex: $P_{\mathbr{\omega}} = \sum_{i=1}^{M} \omega_i P_i$. For the sake of brevity, we omit the dependency on $\pi$ of all the quantities. Moreover, we define the optimal transition model $P_{\mathbr{\omega}^*}$ as the model that maximizes the expected return, \ie $J^{P_{\mathbr{\omega}^*}}_\mu \ge J^{P_{\mathbr{\omega}}}_\mu$ for all $P_{\mathbr{\omega}} \in \mathrm{co}(\bm{P})$. We start by stating some results on the expected relative advantage functions.
\begin{restatable}[]{lemma}{oppositeSignAdv}
	For any transition model $P_{\mathbr{\omega}} \in \mathrm{co}(\bm{P})$ it holds that: $\sum_{i=1}^M \omega_i A_{P_{\mathbr{\omega}}}^{P_i}(s,a) = 0$ for all $s\in\mathcal{S}$ and $a\in\mathcal{A}$.
\end{restatable}\label{thr:oppositeSignAdv}
As a consequence, we observe that also the expected relative advantage functions $\mathds{A}_{P_{\mathbr{\omega}}, \mu}^{P_i}$ sums up to zero when weighted by
the coefficients $\mathbr{\omega}$. An analogous statement holds when the policy is defined as a convex combination of vertex policies. The following theorem establishes an essential property of the optimal transition model.
\begin{restatable}[]{thr}{optimalAdv}
\label{thr:optimalAdv}
	For any transition model $P_{\mathbr{\omega}} \in \mathrm{co}(\bm{P})$ it holds that $\mathds{A}_{P_{{\mathbr{\omega}}^*}, \mu}^{P_{\mathbr{\omega}}} \le 0$. Moreover, for all $P_{\mathbr{\omega}} \in \mathrm{co}\big( \{P_i \in \bm{P} : \omega^*_i > 0 \} \big)$, it holds that $\mathds{A}_{P_{{\mathbr{\omega}}^*}, \mu}^{P_{\mathbr{\omega}}} = 0$.
\end{restatable}
The theorem provides a necessary condition for a transition model to be optimal, \ie all the
expected relative advantages must be non-positive and, moreover, those of the vertex transition 
models associated with non-zero coefficients must be zero.
It is worth noting that the expected relative advantage $\mathds{A}_{P_{\mathbr{\omega}},\mu}^{P_{\mathbr{\omega}'}}$ represents only a \emph{local} measure of the performance improvement, as it is defined by taking the expectation of the relative advantage $A_{P_{\mathbr{\omega}}}^{P_{\mathbr{\omega}'}}(s,a)$ \wrt the current $\delta_{\mu}^{P_{\mathbr{\omega}}}$. On the other hand, the actual performance
improvement $J^{P_{\mathbr{\omega}'}}_\mu - J^{P_{\mathbr{\omega}}}_\mu$ is a \emph{global} measure, being obtained by averaging the relative advantage $A_{P_{\mathbr{\omega}}}^{P_{\mathbr{\omega}'}}(s,a)$ over the new $\delta_{\mu}^{P_{\mathbr{\omega}'}}$ (Theorem~\ref{thr:perfImprovement}). This is intimately related to the \emph{measure mismatch} claim provided in~\cite{kakade2003sample} as
the model expected relative advantage $\mathds{A}_{P_{\mathbr{\omega}}, \mu}^{P_{\mathbr{\omega}^*}}$ might be null even if $J^{P_{\mathbr{\omega}^*}}_\mu > J^{P_{\mathbr{\omega}}}_\mu$,
making SPMI, like CPI and SPI, stop into locally optimal models. Furthermore, it is intuitive to get convinced that asking for a guaranteed performance improvement may prevent from finding the global optimum, as this may require visiting a lower performance region (see Appendix~\ref{apx:C1} for an example).
Nevertheless, we can provide a bound
for the performance gap between a locally optimal model and the global optimal model.
\begin{restatable}[]{prop}{performanceGap}
	Let $P_{\overbar{\mathbr{\omega}}}$ be a transition model having non-positive relative advantage functions
	\wrt the target models. Then:
	\begin{equation*}
		J^{P_{{\mathbr{\omega}}^*}}_\mu - J^{P_{\overbar{\mathbr{\omega}}}}_\mu \le \frac{1}{1-\gamma}  \sup_{s\in\mathcal{S},a\in\mathcal{A}} \max_{i=1,2,...,M} A_{P_{\overbar{\mathbr{\omega}}}}^{P_i}(s,a).
	\end{equation*}
\end{restatable}
From this result we notice that a sufficient condition for a model to be optimal is that $A_{P_{\overbar{\mathbr{\omega}}}}^{P_i}(s,a) = 0$ for all state-action pairs. This is 
a stronger requirement than the maximization of $J^{P_{\overbar{\mathbr{\omega}}}}_\mu$ as it asks
the model to be optimal in \emph{every} state-action pair independently of the initial state distribution $\mu$;\footnote{This is the same difference between a policy that maximizes the value function $V^{\pi}$ in all states and a policy that maximizes the expected return $J^{\pi}$.}  such a model might not exist when considering
a model space $\mathcal{P}$ that does not include all the possible transition models (see Appendix~\ref{apx:C2} for an example).

\subsection{Analogy with Policy Gradient Methods} \label{sec:gradient}
In this section, we elucidate the relationship
between the relative advantage function and the gradient of the expected return.
Let us start by stating the expression of the gradient of the expected return \wrt a 
parametric transition model. This is the equivalent of the Policy Gradient Theorem~\cite{sutton2000policy} for the transition model.
\begin{restatable}[$P$-Gradient Theorem]{thr}{PGradientTheorem}
\label{thr:PGradientTheorem}
	Let $P_{\mathbr{\omega}}$ be a class of parametric stochastic transition models 
	differentiable in $\mathbr{\omega}$, the gradient of the expected return \wrt $\mathbr{\omega}$
	is given by:
	\begin{align*}
		\nabla_{\mathbr{\omega}} J^{P_{\mathbr{\omega}}}_\mu = \int_{\mathcal{S}} \int_{\mathcal{A}} & \delta_{\mu}^{P_{\mathbr{\omega}}} (s,a) \int_{\mathcal{S}} \nabla_{\mathbr{\omega}} P_{\mathbr{\omega}}(s'|s,a) \times \\
		& \times U^{P_{\mathbr{\omega}}}(s,a,s') \mathrm{d}s' \mathrm{d}a \mathrm{d}s.
	\end{align*}
\end{restatable}
Let us now show the connection between $\nabla_{\mathbr{\omega}} J^{P_{\mathbr{\omega}}}_\mu$
and the expected relative advantage functions. This result extends that of~\citet{kakade2003sample}
to multiple parameter updates.
\begin{restatable}[]{prop}{gradientAdv}
	Let $P$ be the current transition model. Let us consider a target model which is a convex combination of the models in $\mathcal{P}$: $\overbar{P} = \sum_{i=1}^{M} \eta_i P_i$ and the update rule:
	\begin{equation*}
		P' = \beta \overbar{P} + (1-\beta) P, \quad \beta \in [ 0 , 1 ].
	\end{equation*}
	Then, the derivative of the expected return of $P'$ \wrt the $\beta$ coefficients evaluated in $P$ is given by:
	\begin{equation*}
		 \frac{\partial J^{P'}_\mu}{\partial \beta} \bigg\rvert_{\beta = 0} = \sum_{i=1}^{M} \eta_i \mathds{A}_{P,\mu}^{P_i}.
	\end{equation*}
\end{restatable}
The proposition provides an interesting interpretation of the expected relative advantage function. Suppose that $P_{\mathbr{\omega}}$ is the current model and we have to choose which target model(s) we should move toward. The local performance improvement, at the first order, is given by
$J^{P'}_\mu - J^{P}_\mu \simeq \frac{\partial J^{P'}_\mu}{\partial \beta} \big\rvert_{\beta = 0}  \beta = \beta \sum_{i=1}^M \eta_i \mathds{A}_{P,\mu}^{P_i}$. Given that $\beta$ will be determined later by maximizing the bound, the local performance improvement is maximized by assigning one to the coefficient of the model yielding the maximal advantage. 
Therefore, the choice of the direction to follow, when considering the greedy target choice, is based on local information only (gradient),
while the step size $\beta$ is obtained by maximizing the bound on the guaranteed performance improvement (safe), as done in~\cite{pirotta2013adaptive}.

\section{Experimental Evaluation}
\label{sec:experimental}
The goal of this section is to show the benefits of configuring
the environment while the policy learning goes on. The experiments are conducted on two explicative domains: the Student-Teacher domain (unconstrained model space) the Racetrack Simulator (parametric model space). We compare different target choices (greedy and persistent, see Section~\ref{sec:targetChoice}) and different \emph{update strategies}. Specifically, SPMI, that adaptively updates policy and model, is compared with some alternative model learning approaches: SPMI-alt(ernated) in which model and policy updates are forced to be alternated,  SPMI-sup that uses a looser bound, obtained from Theorem~\ref{thr:boundUncoupled} by replacing $D^{\star',\star}_{\ev}$ with $D^{\star',\star}_{\infty}$,\footnote{When considering only policy updates, this is equivalent to the bound used in SPI~\cite{pirotta2013safe}.} SPI+SMI\footnote{SMI (Safe Model Iteration) is SPMI without policy updates.} that optimizes policy and model in sequence and SMI+SPI that does the opposite.

\subsection{Student-Teacher domain}

\begin{figure*}[ht]
\begin{center}
	\centerline{\includegraphics[scale=1]{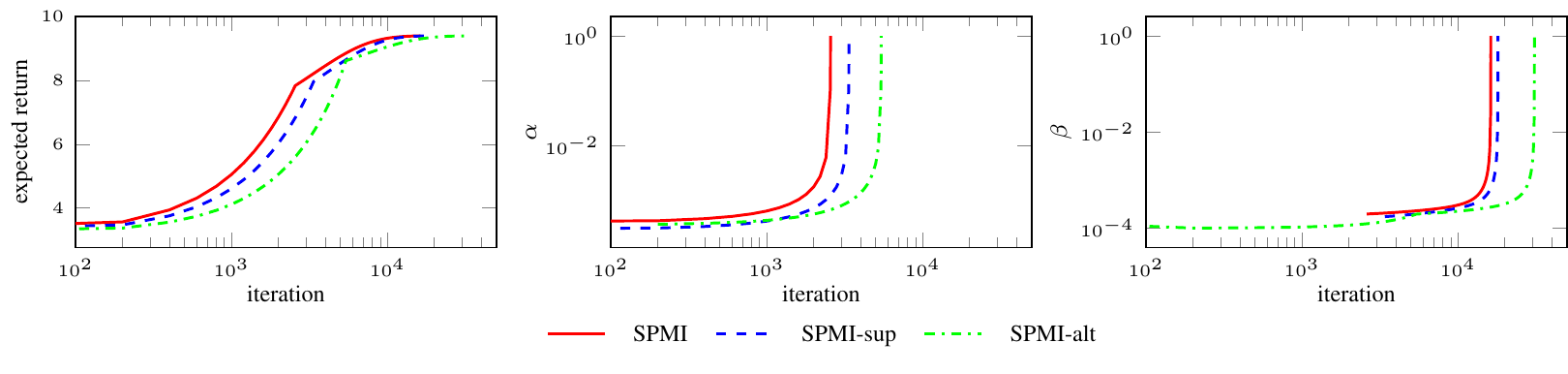}}
	\vskip -0.1in
	\caption{Expected return, $\alpha$ and $\beta$ coefficients for the Student-Teacher domain 2-1-1-2 for different update strategies.}
	\label{fig:2-1-1-2-10-main}
	\end{center}
	\vskip -0.2in
\end{figure*}

\begin{figure*}[ht]
\centering
\begin{minipage}[t]{.33\textwidth}
  \centering
  \includegraphics[scale=1]{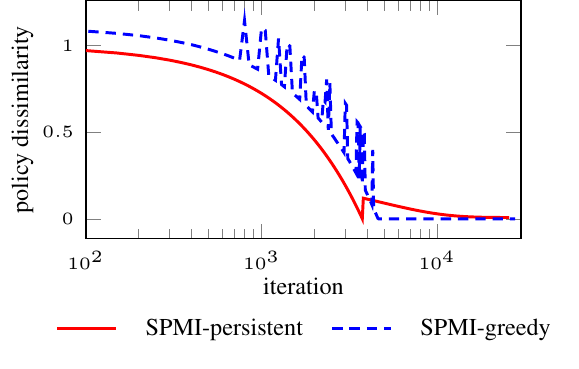}
  \vskip -0.1in
  \captionof{figure}{Policy dissimilarity for greedy and persistent target choices in the 2-1-1-2 case.}\label{fig:2-1-1-2-10-target_}
\end{minipage}%
\hfill
\begin{minipage}[t]{.64\textwidth}
  \centering
  \includegraphics[scale=1]{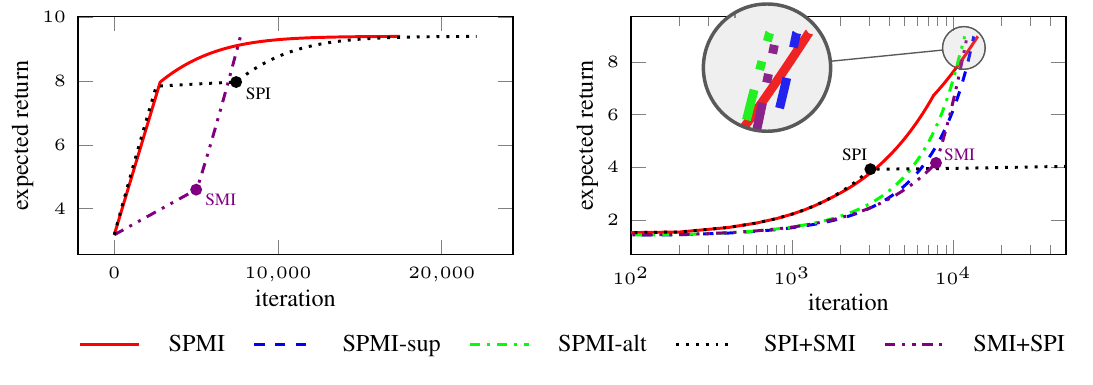}
  \vskip -0.1in
  \captionof{figure}{Expected return for the Student-Teacher domains 2-1-1-2 (left) and 2-3-1-2 (right) for different update strategies.}
  \label{fig:2-1-1-2-10-SPMI-SMI-SPI_}
\end{minipage}
\vskip -0.1in
\end{figure*}

The Student-Teacher domain is a simple model of concept learning, inspired by~\cite{rafferty2011faster}. A student (agent) learns to perform consistent assignments to literals as a result of the statements (\eg \quotes{A+C=3}) provided by an automatic teacher (environment, \eg online platform). The student has a limited policy space as she/he can only change the values of the literals by a finite quantity, but it is possible to configure the difficulty of the teacher's statements, selecting the number of literals in the statement, in order to improve the student's performance (detailed description in Appendix~\ref{apx:student_description}).\footnote{A problem setting is defined by the 4-tuple \emph{number of literals - maximum literal value - maximum update allowed - maximum number of literals in the statement} (\eg 2-1-1-2)}

We start considering the illustrative example in which there are two binary literals, and the student can change only one literal at a time (2-1-1-2). This example aims to illustrate benefits of SPMI over other update strategies and target choices. Further scenarios are reported in Appendix~\ref{apx:student_exp}. In Figure~\ref{fig:2-1-1-2-10-main}, we show the behavior of the different update strategies starting from a uniform initialization. We can see that both SPMI and SPMI-sup perform the policy updates and the model updates in sequence. This is a consequence of the fact that, by looking only at the local advantage function, it is more convenient for the student to learn an almost optimal policy with no intervention on the teacher and then refining the teacher model to gain further reward. The joint and adaptive strategy of SPMI outperforms both SPMI-sup and SPMI-alt. The alternated
model-policy update (SPMI-alt) is not convenient since, with an initial poor-performing policy, updating the model does not yield a significant performance improvement. It is worth noting that all the methods converge in a finite number of steps and the learning rates $\alpha$ and $\beta$ exhibit an exponential growth trend.

In Figure~\ref{fig:2-1-1-2-10-target_}, we compare the greedy target selection with the persistent target selection. The former, while being the best local choice maximizing the advantage, might result in an unstable behavior that slows down the convergence of the algorithm.
In Figure~\ref{fig:2-1-1-2-10-SPMI-SMI-SPI_}, we can notice that learning both policy and model is convenient since the performance of SPMI at convergence is higher than the one of SPI (only policy learned) and SMI (only model learned), corresponding to the markers in Figure~\ref{fig:2-1-1-2-10-SPMI-SMI-SPI_}. Although SPMI adopts the tightest bound, its update strategy is not guaranteed to yield globally the fastest convergence as it is based on local information, \ie expected relative advantage (Figure~\ref{fig:2-1-1-2-10-SPMI-SMI-SPI_} right).

\subsection{Racetrack simulator}
\begin{figure*}[ht]
\centering
\begin{minipage}[t]{.64\textwidth}
  \centering
  \includegraphics[scale=1]{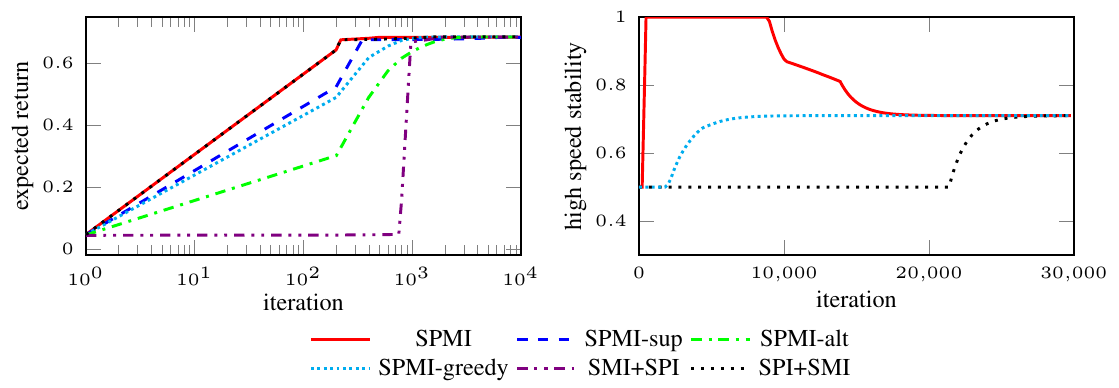}
  \vskip -0.1in
  \captionof{figure}{Expected return and coefficient of the high speed stabiliy vertex model for different update strategies in track T1.}\label{fig:simul2_performance_}
\end{minipage}%
\hfill
\begin{minipage}[t]{.34\textwidth}
  \centering
  \includegraphics[scale=1]{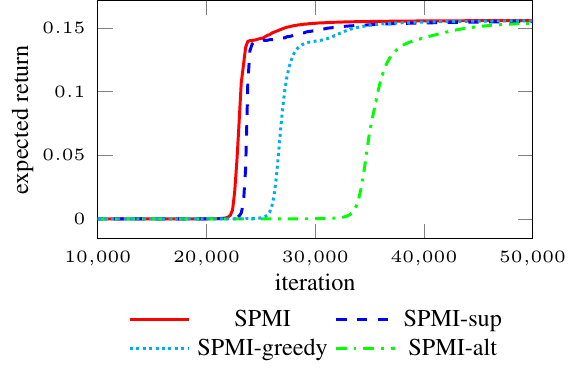}
  \vskip -0.1in
  \captionof{figure}{Expected return in track T2 with 4 vertex models.}
  \label{fig:simul4_performance}
\end{minipage}
\vskip -0.1in
\end{figure*}

The Racetrack simulator is an abstract representation of a car driving problem. The autonomous driver (agent) has to optimize a driving policy to run the vehicle on the track, reaching the finish line as fast as possible. During the process, the agent can configure two vehicle settings to improve her/his driving performance: the \emph{vehicle stability} and the \emph{engine boost} (detailed description in Appendix~\ref{apx:racetrack_description}). We first present an introductory example on a simple track (T1) in which only the vehicle stability can be configured and then we show a case on a different track (T2) including also engine boost configuration. These examples show that the optimal model is not necessarily one of the vertex models. Results on other tracks are reported in Appendix~\ref{apx:racetrack_exp}.

In Figure~\ref{fig:simul2_performance_} left, we highlight the effectiveness of SPMI updates over SPMI-sup and SPMI-alt and sequential executions of SMI and SPI on track T1. Furthermore, the SPMI-greedy, which selects the target greedily in each iteration, results in lower performance \wrt SPMI.
Comparing SPMI with the sequential approaches, we can easily deduce that is not valuable to configure the vehicle stability, \ie updating the model, while the driving policy is still really rough.
Although in the showed example the difference between SPMI and SPI+SMI is way less significant in terms of expected return, their learning paths are quite peculiar. In Figure~\ref{fig:simul2_performance_} right, we show the trend of the model coefficient related to high-speed stability. While the optimal configuration results in a mixed model for vehicle stability, SPMI exploits the maximal high-speed stability to learn the driving policy efficiently in an early stage, SPI+SMI, instead, executes all the policy updates and then directly leads the model to the optimal configuration. SPMI-greedy prefers to avoid the maximal high-speed stability region as well. It is worthwhile to underline that SPMI could temporarily drive the process aside from the optimum if it leads to higher performance from a local perspective. We consider this behavior quite valuable, especially in scenarios where performance degradations during learning are unacceptable.

Figure~\ref{fig:simul4_performance} shows how the previous considerations generalize to an example on a morphologically different track (T2), in which also the engine boost can be configured.
The learning process is characterized by a long exploration phase, both in the model and the policy space, in which the driver cannot lead the vehicle to the finish line to collect any reward. Then, we observe a fast growth in expected return when the agent has acquired enough information to reach the finish line consistently. SPMI displays a more efficient exploration phase compared to other update strategies and target choices, leading the process to a quicker convergence to the optimal model that prefers high speed stability and an intermediate engine boost configuration.

\section{Discussion and Conclusions}
In this paper, we proposed a novel framework (Conf-MDP) to model a set of real-world decision-making scenarios that, from our perspective, have not been covered by the literature so far. In Conf-MDPs the environment dynamics can be partially modified to improve the performance of the learning agent. Conf-MDPs allow modeling many relevant sequential-decision making problems that we believe cannot be effectively addressed using traditional frameworks.

\textbf{Why not a unique agent?}
Representing the environment configurability in the agent model when the environment is under the control of a supervisor is certainly inappropriate. Even when the environment configuration is carried out by the agent, this approach would require the inclusion of \quotes{configuration actions} in the action space to allow the agent to configure the environment directly as a part of the policy optimization. However, in our framework, the environment configuration is performed just once at the beginning of the episode. Moreover, with configuration actions the agent is not really learning a probability distribution on actions, \ie a policy, but a probability distribution on state-state couples, \ie a state kernel. This formulation prevents distinguishing, during the process, the effects of the policy from those of the model, making it difficult to finely constrain the configurations, limit the feasible model space, and recovering, a posteriori, the optimal model-policy pair.

\textbf{Why not a multi-agent system?} When there is no supervisor, the agent is the only learning entity and the environment is completely passive. 
In the presence of a supervisor, it would be misleading to adopt a cooperative multi-agent approach. The supervisor acts externally, at a different level and could be, possibly, totally transparent to the learning agent. Indeed, the supervisor does not operate inside the environment but it is in charge of selecting the most suitable configuration, whereas the agent needs to learn the optimal policy for the given environment.

The second significant contribution of this paper is the formulation of a safe approach, suitable to manage critical tasks, to solve a learning problem in the context of the newly introduced Conf-MDP framework. To this purpose, we proposed a novel tight lower bound on the performance improvement and an algorithm (SPMI) optimizing this bound to learn the policy and the model configuration simultaneously. We then presented an empirical study to show the effectiveness of SPMI in our context and to uphold the introduction of the Conf-MDP framework.

This is a seminal paper on Conf-MDPs and the proposed approach represents only a first step in solving these kinds of problems: many future research directions are open.
Clearly, a first extension could tackle the problem from a sample-based perspective, removing the requirement of knowing the full model space. Furthermore, we could consider different learning approaches, like policy search methods, suitable for continuous state-action spaces.

\clearpage


\bibliography{biblio}

\begin{thebibliography}{33}
\providecommand{\natexlab}[1]{#1}
\providecommand{\url}[1]{\texttt{#1}}
\expandafter\ifx\csname urlstyle\endcsname\relax
  \providecommand{\doi}[1]{doi: #1}\else
  \providecommand{\doi}{doi: \begingroup \urlstyle{rm}\Url}\fi

\bibitem[Abbasi-Yadkori et~al.(2016)Abbasi-Yadkori, Bartlett, and
  Wright]{abbasi2016fast}
Abbasi-Yadkori, Y., Bartlett, P.~L., and Wright, S.~J.
\newblock A fast and reliable policy improvement algorithm.
\newblock In \emph{Artificial Intelligence and Statistics}, pp.\  1338--1346,
  2016.

\bibitem[Achiam et~al.(2017)Achiam, Held, Tamar, and Abbeel]{achiam2017cpo}
Achiam, J., Held, D., Tamar, A., and Abbeel, P.
\newblock Constrained policy optimization.
\newblock In \emph{Proceedings of the 34th International Conference on Machine
  Learning}, volume~70 of \emph{ICML'17}, pp.\  22--31, 2017.

\bibitem[Bowerman(1974)]{bowerman1974nonstationary}
Bowerman, B.~L.
\newblock \emph{Nonstationary Markov decision processes and related topics in
  nonstationary Markov chains}.
\newblock PhD thesis, Iowa State University, 1974.

\bibitem[Bueno et~al.(2017)Bueno, Mau{\'a}, Barros, and
  Cozman]{bueno2017modeling}
Bueno, T.~P., Mau{\'a}, D.~D., Barros, L.~N., and Cozman, F.~G.
\newblock Modeling markov decision processes with imprecise probabilities using
  probabilistic logic programming.
\newblock In \emph{Proceedings of the Tenth International Symposium on
  Imprecise Probability: Theories and Applications}, pp.\  49--60, 2017.

\bibitem[Cheevaprawatdomrong et~al.(2007)Cheevaprawatdomrong, Schochetman,
  Smith, and Garcia]{cheevaprawatdomrong2007solution}
Cheevaprawatdomrong, T., Schochetman, I.~E., Smith, R.~L., and Garcia, A.
\newblock Solution and forecast horizons for infinite-horizon nonhomogeneous
  markov decision processes.
\newblock \emph{Mathematics of Operations Research}, 32\penalty0 (1):\penalty0
  51--72, 2007.

\bibitem[Ciosek \& Whiteson(2017)Ciosek and Whiteson]{ciosek2017offer}
Ciosek, K.~A. and Whiteson, S.
\newblock Offer: Off-environment reinforcement learning.
\newblock In \emph{AAAI}, pp.\  1819--1825, 2017.

\bibitem[Delgado et~al.(2009)Delgado, de~Barros, Cozman, and
  Shirota]{delgado2009representing}
Delgado, K.~V., de~Barros, L.~N., Cozman, F.~G., and Shirota, R.
\newblock Representing and solving factored markov decision processes with
  imprecise probabilities.
\newblock \emph{Proceedings ISIPTA, Durham, United Kingdom}, pp.\  169--178,
  2009.

\bibitem[Deza \& Deza(2009)Deza and Deza]{deza2009encyclopedia}
Deza, M.~M. and Deza, E.
\newblock Encyclopedia of distances.
\newblock In \emph{Encyclopedia of Distances}, pp.\  1--583. Springer, 2009.

\bibitem[Florensa et~al.(2017)Florensa, Held, Wulfmeier, Zhang, and
  Abbeel]{florensa2017reversecurriculum}
Florensa, C., Held, D., Wulfmeier, M., Zhang, M., and Abbeel, P.
\newblock Reverse curriculum generation for reinforcement learning.
\newblock In \emph{Conference on Robot Learning}, pp.\  482--495, 2017.

\bibitem[Garcia \& Smith(2000)Garcia and Smith]{garcia2000solving}
Garcia, A. and Smith, R.~L.
\newblock Solving nonstationary infinite horizon dynamic optimization problems.
\newblock \emph{Journal of Mathematical Analysis and Applications},
  244\penalty0 (2):\penalty0 304--317, 2000.

\bibitem[Ghate \& Smith(2013)Ghate and Smith]{ghate2013linear}
Ghate, A. and Smith, R.~L.
\newblock A linear programming approach to nonstationary infinite-horizon
  markov decision processes.
\newblock \emph{Operations Research}, 61\penalty0 (2):\penalty0 413--425, 2013.

\bibitem[Ghavamzadeh et~al.(2016)Ghavamzadeh, Petrik, and
  Chow]{ghavamzadeh2016safe}
Ghavamzadeh, M., Petrik, M., and Chow, Y.
\newblock Safe policy improvement by minimizing robust baseline regret.
\newblock In \emph{Advances in Neural Information Processing Systems}, pp.\
  2298--2306, 2016.

\bibitem[Givan et~al.(1997)Givan, Leach, and Dean]{givan1997boundedmdp}
Givan, R., Leach, S., and Dean, T.
\newblock Bounded parameter markov decision processes.
\newblock In Steel, S. and Alami, R. (eds.), \emph{Recent Advances in AI
  Planning}, pp.\  234--246. Springer Berlin Heidelberg, 1997.

\bibitem[Haviv \& Van~der Heyden(1984)Haviv and Van~der
  Heyden]{haviv1984perturbation}
Haviv, M. and Van~der Heyden, L.
\newblock Perturbation bounds for the stationary probabilities of a finite
  markov chain.
\newblock \emph{Advances in Applied Probability}, 16\penalty0 (4):\penalty0
  804--818, 1984.

\bibitem[Holland \& Goldberg(1989)Holland and Goldberg]{holland1989genetic}
Holland, J. and Goldberg, D.
\newblock Genetic algorithms in search, optimization and machine learning.
\newblock \emph{Massachusetts: Addison-Wesley}, 1989.

\bibitem[Hopp et~al.(1987)Hopp, Bean, and Smith]{hopp1987new}
Hopp, W.~J., Bean, J.~C., and Smith, R.~L.
\newblock A new optimality criterion for nonhomogeneous markov decision
  processes.
\newblock \emph{Operations Research}, 35\penalty0 (6):\penalty0 875--883, 1987.

\bibitem[Kakade \& Langford(2002)Kakade and Langford]{kakade2002approximately}
Kakade, S. and Langford, J.
\newblock Approximately optimal approximate reinforcement learning.
\newblock In \emph{Proceedings of the 19th International Conference on Machine
  Learning}, volume~2, pp.\  267--274, 2002.

\bibitem[Kakade et~al.(2003)]{kakade2003sample}
Kakade, S.~M. et~al.
\newblock \emph{On the sample complexity of reinforcement learning}.
\newblock PhD thesis, University of London London, England, 2003.

\bibitem[Ni \& Liu(2008)Ni and Liu]{ni2008boundedpomdp}
Ni, Y. and Liu, Z.-Q.
\newblock Bounded-parameter partially observable markov decision processes.
\newblock In \emph{Proceedings of the Eighteenth International Conference on
  International Conference on Automated Planning and Scheduling}, pp.\
  240--247. AAAI Press, 2008.

\bibitem[Osogami(2015)]{osogami2015robustpomdp}
Osogami, T.
\newblock Robust partially observable markov decision process.
\newblock In \emph{Proceedings of the 32nd International Conference on Machine
  Learning}, volume~37 of \emph{ICML'15}, pp.\  106--115, 2015.

\bibitem[Papini et~al.(2017)Papini, Pirotta, and Restelli]{papini2017adaptive}
Papini, M., Pirotta, M., and Restelli, M.
\newblock Adaptive batch size for safe policy gradients.
\newblock In Guyon, I., Luxburg, U.~V., Bengio, S., Wallach, H., Fergus, R.,
  Vishwanathan, S., and Garnett, R. (eds.), \emph{Advances in Neural
  Information Processing Systems 30}, pp.\  3594--3603. Curran Associates,
  Inc., 2017.

\bibitem[Peters et~al.(2010)Peters, M{\"u}lling, and Altun]{peters2010relative}
Peters, J., M{\"u}lling, K., and Altun, Y.
\newblock Relative entropy policy search.
\newblock In \emph{AAAI}, pp.\  1607--1612. Atlanta, 2010.

\bibitem[Pirotta et~al.(2013{\natexlab{a}})Pirotta, Restelli, and
  Bascetta]{pirotta2013adaptive}
Pirotta, M., Restelli, M., and Bascetta, L.
\newblock Adaptive step-size for policy gradient methods.
\newblock In \emph{Advances in Neural Information Processing Systems}, pp.\
  1394--1402, 2013{\natexlab{a}}.

\bibitem[Pirotta et~al.(2013{\natexlab{b}})Pirotta, Restelli, Pecorino, and
  Calandriello]{pirotta2013safe}
Pirotta, M., Restelli, M., Pecorino, A., and Calandriello, D.
\newblock Safe policy iteration.
\newblock In \emph{Proceedings of the 30th International Conference on
  International Conference on Machine Learning}, volume~28 of \emph{ICML'13},
  pp.\  307--315, 2013{\natexlab{b}}.

\bibitem[Puterman(2014)]{puterman2014markov}
Puterman, M.~L.
\newblock \emph{Markov decision processes: discrete stochastic dynamic
  programming}.
\newblock John Wiley \& Sons, 2014.

\bibitem[Rafferty et~al.(2011)Rafferty, Brunskill, Griffiths, and
  Shafto]{rafferty2011faster}
Rafferty, A.~N., Brunskill, E., Griffiths, T.~L., and Shafto, P.
\newblock Faster teaching by pomdp planning.
\newblock In \emph{AIED}, pp.\  280--287. Springer, 2011.

\bibitem[Rubinstein(1999)]{rubinstein1999cross}
Rubinstein, R.
\newblock The cross-entropy method for combinatorial and continuous
  optimization.
\newblock \emph{Methodology and computing in applied probability}, 1\penalty0
  (2):\penalty0 127--190, 1999.

\bibitem[Satia \& Lave~Jr(1973)Satia and Lave~Jr]{satia1973markovian}
Satia, J.~K. and Lave~Jr, R.~E.
\newblock Markovian decision processes with uncertain transition probabilities.
\newblock \emph{Operations Research}, 21\penalty0 (3):\penalty0 728--740, 1973.

\bibitem[Schulman et~al.(2015)Schulman, Levine, Abbeel, Jordan, and
  Moritz]{schulman2015trust}
Schulman, J., Levine, S., Abbeel, P., Jordan, M., and Moritz, P.
\newblock Trust region policy optimization.
\newblock In \emph{Proceedings of the 32nd International Conference on Machine
  Learning}, ICML'15, pp.\  1889--1897, 2015.

\bibitem[Schulman et~al.(2017)Schulman, Wolski, Dhariwal, Radford, and
  Klimov]{schulman2017proximal}
Schulman, J., Wolski, F., Dhariwal, P., Radford, A., and Klimov, O.
\newblock Proximal policy optimization algorithms.
\newblock \emph{arXiv preprint arXiv:1707.06347}, 2017.

\bibitem[Sutton \& Barto(1998)Sutton and Barto]{sutton1998reinforcement}
Sutton, R.~S. and Barto, A.~G.
\newblock \emph{Reinforcement learning: An introduction}, volume~1.
\newblock MIT press Cambridge, 1998.

\bibitem[Sutton et~al.(2000)Sutton, McAllester, Singh, and
  Mansour]{sutton2000policy}
Sutton, R.~S., McAllester, D.~A., Singh, S.~P., and Mansour, Y.
\newblock Policy gradient methods for reinforcement learning with function
  approximation.
\newblock In \emph{Advances in Neural Information Processing Systems}, pp.\
  1057--1063, 2000.

\bibitem[White~III \& Eldeib(1994)White~III and Eldeib]{white1994markov}
White~III, C.~C. and Eldeib, H.~K.
\newblock Markov decision processes with imprecise transition probabilities.
\newblock \emph{Operations Research}, 42\penalty0 (4):\penalty0 739--749, 1994.

\end{thebibliography}
\bibliographystyle{icml2018}

\onecolumn
\appendix

\section{Proofs and Derivations}
\label{apx:proofs}

\distributionsBoundCoupled*
\begin{proof}

	Exploiting the recursive equation of the $\gamma$-discounted state distribution ($\ref{eq:d_mu}$) we can write the distributions difference as follows:

	\begin{align}
		d^{P',\pi'}_\mu(s) - d^{P,\pi}_\mu(s) & = (1 - \gamma) \mu(s) + \gamma \int_{\mathcal{S}} d^{P',\pi'}_\mu(s') P'^{\pi'}(s|s') \mathrm{d}s' - (1 - \gamma) \mu(s) - \gamma \int_{\mathcal{S}} d^{P,\pi}_\mu(s') P^{\pi}(s|s') \mathrm{d}s' = \nonumber \\
		& = \gamma \int_{\mathcal{S}} d^{P',\pi'}_\mu(s') P'^{\pi'}(s|s') \mathrm{d}s' - \gamma \int_{\mathcal{S}} d^{P,\pi}_\mu(s') P^{\pi}(s|s') \mathrm{d}s'  \pm \gamma \int_{\mathcal{S}} d^{P,\pi}_\mu(s')  P'^{\pi'}(s|s') \mathrm{d}s' = \nonumber \\
		& = \gamma \int_{\mathcal{S}} \big( d^{P',\pi'}_\mu(s') - d^{P,\pi}_\mu(s') \big) P'^{\pi'}(s|s') \mathrm{d}s'  +  \gamma \int_{\mathcal{S}} d^{P,\pi}_\mu(s') \big( P'^{\pi'}(s|s') - P^{\pi}(s|s') \big) \mathrm{d}s'.		\label{line:d1}
	\end{align}

	Applying the 	$\ell^1$-norm to the equation $(\ref{line:d1})$ we can state the following:
	\begin{align}
		\Big\| d^{P',\pi'}_\mu - d^{P,\pi}_\mu \Big\|_1  & \leq  \gamma \int_{\mathcal{S}} \bigg| \int_{\mathcal{S}} \big( d^{P',\pi'}_\mu(s') - d^{P,\pi}_\mu(s') \big) P'^{\pi'}(s|s') \mathrm{d}s' \bigg| \mathrm{d}s  +  \gamma  \int_{\mathcal{S}} \bigg|  \int_{\mathcal{S}} d^{P,\pi}_\mu(s') \big( P'^{\pi'}(s|s') - P^{\pi}(s|s') \big) \mathrm{d}s'  \bigg| \mathrm{d}s \le \label{line:d2}\\
		& \le \gamma  \int_{\mathcal{S}} \Big| d^{P',\pi'}_\mu(s') - d^{P,\pi}_\mu(s') \Big| \int_{\mathcal{S}} P'^{\pi'}(s|s') \mathrm{d}s \mathrm{d}s'   +  \gamma \int_{\mathcal{S}} d^{P,\pi}_\mu(s')  \int_{\mathcal{S}} \Big| P'^{\pi'}(s|s') - P^{\pi}(s|s') \Big| \mathrm{d}s \mathrm{d}s' \le \label{line:d3} \\
		& \le \gamma  \Big\| d^{P',\pi'}_\mu  -   d^{P,\pi}_\mu \Big\|_{1}  +  \gamma \ev_{s\sim d_{\mu}^{P,\pi}} \Big\| {P'}^{\pi'}(\cdot|s) - P^{\pi}(\cdot|s) \Big\|_1 \label{line:d4} \le\\
		& \leq \frac{\gamma}{1 - \gamma} \ev_{s\sim d_{\mu}^{P,\pi}} \Big\| {P'}^{\pi'}(\cdot|s) - P^{\pi}(\cdot|s) \Big\|_1 = D_{\ev}^{{P'}^{\pi'}, P^{\pi}}. \notag
	\end{align}
	In (\ref{line:d2}) we exploited the subadditivity of the norm $\|x+y\| \le \|x\| + \|y\|$ and (\ref{line:d4}) derives from (\ref{line:d3}) by observing that $\int_{\mathcal{S}} P'^{\pi'}(s|s') \mathrm{d}s = 1$.

\end{proof}

\distributionsBound*
\begin{proof}
	We prove this corollary by decomposing the expression $\ev_{s\sim d_{\mu}^{P,\pi}} \Big\| {P'}^{\pi'}(\cdot|s) - P^{\pi}(\cdot|s) \Big\|_1 $:
	\begin{align*}
		 \Big\| {P'}^{\pi'}(\cdot|s) - P^{\pi}(\cdot|s) \Big\|_1 = & \int_{\mathcal{S}} \Big| P'^{\pi'}(s'|s) - P^{\pi}(s'|s) \Big| \mathrm{d}s' = \\
		 & = \int_{\mathcal{S}} \Big| \int_{\mathcal{A}} \Big( P'(s'|s,a)\pi'(a|s) - P(s'|s,a)\pi(a|s) \Big) \mathrm{d}a \Big| \mathrm{d}s' = \\
		 & = \int_{\mathcal{S}} \Big| \int_{\mathcal{A}} \Big( P'(s'|s,a)\pi'(a|s) - P(s'|s,a)\pi(a|s) \pm P'(s'|s,a)\pi(a|s) \Big) \mathrm{d}a \Big| \mathrm{d}s' = \\
		 & =  \int_{\mathcal{S}} \Big| \int_{\mathcal{A}} \Big( \big(\pi'(a|s) - \pi(a|s) \big) P'(s'|s,a) + \pi(a|s') \big(P'(s'|s,a) -  P(s'|s,a) \big) \Big) \mathrm{d}a \Big| \mathrm{d}s' \le \\
		 & \le \int_{\mathcal{S}} \int_{\mathcal{A}} \Big| \pi'(a|s) - \pi(a|s) \Big| P'(s'|s,a) \mathrm{d}a \mathrm{d}s' +  \int_{\mathcal{S}} \int_{\mathcal{A}} \pi(a|s) \Big| P'(s'|s,a) -  P(s'|s,a) \Big| \mathrm{d}a \mathrm{d}s'.
	\end{align*}
	We now take the expectation \wrt $d_{\mu}^{P,\pi}$ and exploit the monotonicity property:
	\begin{align}
		\ev_{s\sim d_{\mu}^{P,\pi}} \Big\| {P'}^{\pi'}(\cdot|s) - P^{\pi}(\cdot|s) \Big\|_1 & \le 
			\int_{\mathcal{S}}  d_{\mu}^{P,\pi}(s) \int_{\mathcal{S}} \int_{\mathcal{A}} \Big| \pi'(a|s) - \pi(a|s) \Big| P'(s'|s,a) \mathrm{d}a \mathrm{d}s \mathrm{d}s'+ \notag \\
			& \quad + \int_{\mathcal{S}}  d_{\mu}^{P,\pi}(s) \int_{\mathcal{S}} \int_{\mathcal{A}} \pi(a|s) \Big| P'(s'|s,a) -  P(s'|s,a) \Big| \mathrm{d}a \mathrm{d}s \mathrm{d}s' \le \label{eq4} \\
			& \le \int_{\mathcal{S}}  d_{\mu}^{P,\pi}(s)\int_{\mathcal{A}} \Big| \pi'(a|s) - \pi(a|s) \Big| \mathrm{d}a \mathrm{d}s' + \notag \\
			& \quad + \int_{\mathcal{S}}   \int_{\mathcal{A}} \delta_{\mu}^{P,\pi}(s,a) \int_{\mathcal{S}} \Big| P'(s'|s,a) -  P(s'|s,a) \Big| \mathrm{d}s \mathrm{d}a \mathrm{d}s' = \label{eq5}\\
			& = \ev_{s\sim d_{\mu}^{P,\pi}} \Big\| \pi'(\cdot|s) - \pi(\cdot|s) \Big\|_1 + \ev_{(s,a)\sim \delta_{\mu}^{P,\pi}} \big\| P'(\cdot|s,a) - P(\cdot|s,a) \big\|_1 = D_{\ev}^{\pi', \pi} + D_{\ev}^{P', P},\notag		
	\end{align}
	where (\ref{eq5}) follows from (\ref{eq4}) by observing that $d_{\mu}^{P,\pi}(s)  \pi(a|s') = \delta_{\mu}^{P,\pi}(s,a)$.
\end{proof}

\perfImprovement*
\begin{proof}
	Let us start from the definition of $J^{P',\pi'}_\mu$:
	\begin{align}
		(1-\gamma) J^{P',\pi'}_\mu & = \int_{\mathcal{S}} \int_{\mathcal{A}} d_\mu^{P',\pi'}(s) \pi'(a|s)R(s,a) \mathrm{d}a \mathrm{d}s = \nonumber \\
		\label{line:impr1}
        & = \int_{\mathcal{S}}  d_\mu^{P',\pi'}(s) \int_{\mathcal{A}} \pi'(a|s)  R(s,a)  \mathrm{d}a \mathrm{d}s \pm \int_{\mathcal{S}}  d_\mu^{P',\pi'}(s) V^{P,\pi}(s) \mathrm{d}s = \\ 
        \label{line:impr2}
        & = \int_{\mathcal{S}}  d_\mu^{P',\pi'}(s) \int_{\mathcal{A}} \pi'(a|s) R(s,a)  \mathrm{d}a \mathrm{d}s + \\ 
        & \quad +  \int_{\mathcal{S}} \bigg( (1-\gamma)\mu(s') + \gamma \int_{\mathcal{S}} \int_{\mathcal{A}} d_\mu^{P',\pi'}(s) \pi'(a|s) P'(s'|s,a)  \mathrm{d}a \mathrm{d}s \bigg) V^{P,\pi}(s') \mathrm{d}s' - \int_{\mathcal{S}}  d_\mu^{P',\pi'}(s) V^{P,\pi}(s) \mathrm{d}s =  \nonumber \\
        & = \int_{\mathcal{S}}  d_\mu^{P',\pi'}(s) \bigg( \int_{\mathcal{A}} \pi'(a|s)   \int_{\mathcal{S}} P'(s'|s,a) \Big( R(s,a) + \gamma V^{P,\pi}(s') \Big) \mathrm{d}s'   \mathrm{d}a - V^{P,\pi}(s) \bigg) \mathrm{d}s + \label{line:impr3} \\
        & \quad + (1-\gamma) \int_{\mathcal{S}} \mu(s') V^{P,\pi}(s') \mathrm{d}s'= \nonumber \\
        & =  \int_{\mathcal{S}} d_\mu^{P',\pi'}(s) A_{P,\pi}^{P',\pi'}(s) \mathrm{d}s + (1-\gamma) J_\mu^{P,\pi},\label{line:impr4}
    \end{align}
	where we have exploited the recursive formulation of $d_\mu^{P',\pi'}$ (\ref{eq:d_mu}) to rewrite (\ref{line:impr1}) into (\ref{line:impr2}) and~\eqref{line:impr4} follows from~\eqref{line:impr3} by observing that $\int_{\mathcal{S}} \mu(s') V^{P,\pi}(s') \mathrm{d}s' = J_\mu^{P,\pi}$ and using the definition $U^{P,\pi}(s,a,s') =  R(s,a) + \gamma V^{P,\pi}(s')$.
\end{proof}

\boundCoupled*
\begin{proof}
		
	Exploiting the bounds on the $\gamma$-discounted state distributions difference (Proposition $\ref{thr:distributionsBoundCoupled}$) we can easily attain the performance improvement bound:
	
	\begin{align}
		J^{P',\pi'}_{\mu} - J^{P,\pi}_{\mu} & = \frac{1}{1-\gamma} \int_{\mathcal{S}} d_{\mu}^{P',\pi'}(s) A_{P,\pi}^{P',\pi'}(s) \mathrm{d}s = \notag\\
		& = \frac{1}{1-\gamma} \int_{\mathcal{S}} d_{\mu}^{P,\pi}(s) A_{P,\pi}^{P',\pi'}(s) \mathrm{d}s +  \frac{1}{1-\gamma} \int_{\mathcal{S}} \big( d_{\mu}^{P',\pi'}(s) - d_{\mu}^{P,\pi}(s) \big) A_{P,\pi}^{P',\pi'}(s) \mathrm{d}s  \ge \label{line:cb1}\\
		& \ge \frac{\mathds{A}_{P,\pi,\mu}^{P',\pi'}}{1-\gamma} - \frac{1}{1-\gamma} \bigg|  \int_{\mathcal{S}} \big( d_{\mu}^{P',\pi'}(s) - d_{\mu}^{P,\pi}(s) \big) A_{P,\pi}^{P',\pi'}(s) \mathrm{d}s \bigg| \ge \label{line:cb2}\\
		& \geq   \frac{\mathds{A}_{{P,\pi},\mu}^{P',\pi'}}{1 - \gamma} -  \frac{1}{1 - \gamma}  \big\| d^{P',\pi'}_{\mu} - d^{P,\pi}_{\mu} \big\|_{1} \frac{\Delta A_{P,\pi}^{P',\pi'}}{2} \ge \label{line:cb3}\\
		& \geq   \frac{\mathds{A}_{{P,\pi},\mu}^{P',\pi'}}{1 - \gamma} -  \frac{\gamma}{(1 - \gamma)^{2}}  \ev_{s\sim d_{\mu}^{P,\pi}} \big\| {P'}^{\pi'} - P^{\pi} \big\|_1  \frac{\Delta A_{P,\pi}^{P',\pi'}}{2}, \label{line:cb4}
	\end{align}	
	where~\eqref{line:cb2} follows from~\eqref{line:cb1} by observing that $b \ge -|b|$, line~\eqref{line:cb3} follows from~\eqref{line:cb2} by applying Corollary 2.4 of~\cite{haviv1984perturbation} and~\eqref{line:cb4} is obtained by using Corollary~\ref{thr:distributionsBoundCoupled}.
\end{proof}

\begin{lemma}
\label{thr:lemmaAdvantage}
	The following equality relates the joint relative advantage function and the relative advantage functions:
	\begin{equation*}
		A_{P,\pi}^{P',\pi'}(s) = A_{P,\pi}^{P,\pi'}(s) + \int_{\mathcal{A}} \pi'(a|s) A_{P,\pi}^{P',\pi}(s,a) \mathrm{d}a.
	\end{equation*}
\end{lemma} 

\begin{proof}
	\begin{align}
		A_{P,\pi}^{P',\pi'}(s) & = \int_{\mathcal{A}} \int_{\mathcal{S}} \pi'(a|s)  P'(s'|s,a) U^{P,\pi}(s,a,s') \mathrm{d}s'\mathrm{d}a - V^{P,\pi}(s) = \notag \\
			& =\int_{\mathcal{A}} \int_{\mathcal{S}} \pi'(a|s)  P'(s'|s,a) U^{P,\pi}(s,a,s') \mathrm{d}s'\mathrm{d}a - V^{P,\pi}(s)\pm \int_{\mathcal{A}} \int_{\mathcal{S}} \pi'(a|s)  P(s'|s,a) U^{P,\pi}(s,a,s') \mathrm{d}s'\mathrm{d}a = \notag \\
			& = \int_{\mathcal{A}} \int_{\mathcal{S}} \pi'(a|s)  P(s'|s,a) U^{P,\pi}(s,a,s') \mathrm{d}s'\mathrm{d}a - V^{P,\pi}(s)+ \notag \\
			& \quad +  \int_{\mathcal{A}} \int_{\mathcal{S}} \pi'(a|s)  \Big( P'(s'|s,a) - P(s'|s,a) \Big) U^{P,\pi}(s,a,s') \mathrm{d}s'\mathrm{d}a = \notag \\
			& = \int_{\mathcal{A}} \pi'(a|s)  Q^{P,\pi}(s,a) \mathrm{d}a - V^{P,\pi}(s)  + \int_{\mathcal{A}}  \pi'(a|s) \int_{\mathcal{S}} \Big( P'(s'|s,a) - P(s'|s,a) \Big) U^{P,\pi}(s,a,s') \mathrm{d}s'\mathrm{d}a = \label{line:1}\\
			& = A_{P,\pi}^{P,\pi'}(s) + \int_{\mathcal{A}} \pi'(a|s) A_{P,\pi}^{P',\pi}(s,a) \mathrm{d}a, \label{line:2}
	\end{align}
	where we have applied in the first addendum (\ref{line:1}) the definition of $A_{P,\pi}^{P,\pi'}(s)$ observing that $A_{P,\pi}^{P,\pi'}(s) = \int_{\mathcal{A}} \pi'(a|s)A^{P,\pi}(s,a) \mathrm{d}a = \int_{\mathcal{A}} \pi'(a|s)\big( Q^{P,\pi}(s,a) - V^{P,\pi}(s) \big)$ to get (\ref{line:2}); and similarly for the second addendum of (\ref{line:1}) the fact that $A_{P,\pi}^{P',\pi}(s,a) = \int_{\mathcal{S}} P'(s'|s,a) A^{P,\pi}(s,a,s') \mathrm{d}s' = \int_{\mathcal{S}} P'(s'|s,a) \big( U^{P,\pi}(s,a,s') - Q^{P,\pi}(s,a) \big)\mathrm{d}s'$.
\end{proof}

\begin{lemma}
\label{thr:boundAdvantage}
	The following bound relates the joint relative advantage function and the relative advantage functions:
	\begin{equation*}
		\bigg| \mathds{A}_{P,\pi,\mu}^{P',\pi'} -\Big(\mathds{A}_{P,\pi,\mu}^{P',\pi} +  \mathds{A}_{P,\pi,\mu}^{P,\pi'}\Big) \bigg| \le \gamma D^{\pi', \pi}_{\ev} D^{P',P}_{\infty} \frac{\Delta Q^{P,\pi}}{2},
	\end{equation*}
	where $D^{P',P}_{\infty}= \sup_{s\in\mathcal{S},a\in\mathcal{A}} \big\| P'(\cdot|s,a) - P(\cdot|s,a) \big\|_1$ and $\Delta Q^{P,\pi} = \sup_{s,s'\in\mathcal{S}, a,a'\in\mathcal{A}} \big| Q^{P,\pi}(s',a') - Q^{P,\pi}(s,a) \big|$.
\end{lemma} 

\begin{proof}

	We can rewrite the expected relative advantage $\mathds{A}_{{P,\pi},\mu}^{P',\pi'}$ exploiting the definition:

	\begin{align}
		\mathds{A}_{{P,\pi},\mu}^{P',\pi'} & = \int_{\mathcal{S}} d_{\mu}^{P,\pi} A_{P,\pi}^{P',\pi'}(s) \mathrm{d}s =  \nonumber \\
		& = \int_{\mathcal{S}} d_{\mu}^{P,\pi}(s) \Big( A_{P,\pi}^{P,\pi'}(s) + \int_{\mathcal{A}} \pi'(a|s) A_{P,\pi}^{P',\pi}(s,a) \mathrm{d}a \Big) \mathrm{d}s =\nonumber \\
		& = \int_{\mathcal{S}} d_{\mu}^{P,\pi}(s) A_{P,\pi}^{P,\pi'}(s) \mathrm{d}s + \int_{\mathcal{S}} \int_{\mathcal{A}} d_{\mu}^{P,\pi} \pi(a|s) A_{P,\pi}^{P',\pi}(s,a) \mathrm{d}a \mathrm{d}s + \nonumber \\
		& \quad + \int_{\mathcal{S}} d_{\mu}^{P,\pi}(s)  \int_{\mathcal{A}} \big( \pi'(a|s) - \pi(a|s) \big) A_{P,\pi}^{P',\pi}(s,a) \mathrm{d}a \mathrm{d}s = \nonumber \\
		& = \mathds{A}_{{P,\pi},\mu}^{P,\pi'} +  \mathds{A}_{{P,\pi},\mu}^{P',\pi} + \int_{\mathcal{S}} d_{\mu}^{P,\pi}(s)  \int_{\mathcal{A}} \big( \pi'(a|s) - \pi(a|s) \big) A_{P,\pi}^{P',\pi}(s,a) \mathrm{d}a \mathrm{d}s. \label{line:b1}
	\end{align}
	
	From equation $(\ref{line:b1})$ we can straightforwardly state the following inequalities:

	\begin{align*}
		\mathds{A}_{{P,\pi},\mu}^{P',\pi'} & \geq \mathds{A}_{{P,\pi},\mu}^{P,\pi'} +  \mathds{A}_{{P,\pi},\mu}^{P',\pi} - \bigg| \int_{\mathcal{S}} d_{\mu}^{P,\pi}(s) \int_{\mathcal{A}} \big( \pi'(a|s) - \pi(a|s) \big) A_{P,\pi}^{P',\pi}(s,a) \mathrm{d}a \mathrm{d}s  \bigg|, \\
		\mathds{A}_{{P,\pi},\mu}^{P',\pi'} & \leq \mathds{A}_{{P,\pi},\mu}^{P,\pi'} +  \mathds{A}_{{P,\pi},\mu}^{P',\pi} + \bigg| \int_{\mathcal{S}} d_{\mu}^{P,\pi}(s) \int_{\mathcal{A}} \big( \pi'(a|s) - \pi(a|s) \big) A_{P,\pi}^{P',\pi}(s,a) \mathrm{d}a \mathrm{d}s  \bigg|,
	\end{align*}
	
	then we bound the right hand side:
	
	\begin{align*}
		\Big| \mathds{A}_{{P,\pi},\mu}^{P',\pi'} - (\mathds{A}_{{P,\pi},\mu}^{P,\pi'}  +  \mathds{A}_{{P,\pi},\mu}^{P',\pi}) \Big| & \leq \bigg| \int_{\mathcal{S}} d_{\mu}^{P,\pi}(s)  \int_{\mathcal{A}} \big( \pi'(a|s) - \pi(a|s) \big) A_{P,\pi}^{P',\pi}(s,a) \mathrm{d}a \mathrm{d}s  \bigg| \le \\
		& \leq \int_{\mathcal{S}} \int_{\mathcal{A}} d_{\mu}^{P,\pi}(s) \Big| \big( \pi'(a|s) - \pi(a|s) \big)  A_{P,\pi}^{P',\pi}(s,a) \Big| \mathrm{d}a \mathrm{d}s \le \\
		& \leq \int_{\mathcal{S}}  d_{\mu}^{P,\pi}(s) \int_{\mathcal{A}} \Big| \pi'(a|s) - \pi(a|s) \Big| \mathrm{d}a \mathrm{d}s  \frac{\Delta A_{P,\pi}^{P',\pi}}{2} \le \\
		& \leq \ev_{s\sim d_{\mu}^{P,\pi} } \big\| \pi'(\cdot|s) - \pi(\cdot|s) \big\|_{1} \frac{\Delta A_{P,\pi}^{P',\pi}}{2} \le \\
		& \le D_{\ev}^{\pi',\pi} \frac{\Delta A_{P,\pi}^{P',\pi}}{2}.
	\end{align*}
	We conclude by bounding the term $\frac{\Delta A_{P,\pi}^{P',\pi}}{2}$:
	\begin{align}
		\frac{\Delta A_{P,\pi}^{P',\pi}}{2} & \le \Big\| \Delta A_{P,\pi}^{P',\pi} \Big\|_{\infty} \le \notag \\
		& \le \sup_{s\in\mathcal{S},a\in\mathcal{A}}  \int_{\mathcal{S}} \Big( P'(s'|s,a) - P(s'|s,a) \Big) U^{P,\pi}(s,a,s') \mathrm{d}s' \le \label{line:adv1} \\
		& \le \gamma \sup_{s\in\mathcal{S},a\in\mathcal{A}}  \int_{\mathcal{S}} \Big( P'(s'|s,a) - P(s'|s,a) \Big) V^{P,\pi}(s') \mathrm{d}s' \le \label{line:adv2}\\
		& \le \gamma \sup_{s\in\mathcal{S},a\in\mathcal{A}} \big\| P'(\cdot|s,a) - P(\cdot|s,a) \big\|_1 \frac{\Delta V^{P,\pi}}{2} \le \label{line:adv3}\\
		& \le \gamma D_{\infty}^{P',P} \frac{\Delta Q^{P,\pi}}{2}\notag,
	\end{align}
	where~\eqref{line:adv2} follows from~\eqref{line:adv1} by observing that $ \int_{\mathcal{S}} \big( P'(s'|s,a) - P(s'|s,a) \big) U^{P,\pi}(s,a,s') \mathrm{d}s' =  \int_{\mathcal{S}} \big( P'(s'|s,a) - P(s'|s,a) \big) \big( R(s,a) + \gamma V^{P,\pi}(s') \big) \mathrm{d}s'$ and~\eqref{line:adv3} is obtained by observing that $\Delta V^{P,\pi} \le \Delta Q^{P,\pi}$. Putting all together we get the lemma.
\end{proof}

\begin{lemma}
\label{thr:boundDeltaAdvantage}
	Let $(P,\pi)$ and $(P',\pi')$ be two model-policy pairs, it holds that:
	\begin{equation*}
		\frac{\Delta A_{P,\pi}^{P',\pi'}}{2}  \leq  (D_{\infty}^{\pi', \pi} + \gamma D_{\infty}^{P', P}) \frac{\Delta Q^{P,\pi}}{2}, \\
	\end{equation*}
	where $D^{P',P}_{\infty}= \sup_{s\in\mathcal{S},a\in\mathcal{A}} \big\| P'(\cdot|s,a) - P(\cdot|s,a) \big\|_1$, $D^{\pi',\pi}_{\infty}= \sup_{s\in\mathcal{S}} \big\| \pi'(\cdot|s) - \pi(\cdot|s) \big\|_1$ and $\Delta Q^{P,\pi} = \sup_{s,s'\in\mathcal{S}, a,a'\in\mathcal{A}} \big| Q^{P,\pi}(s',a') - Q^{P,\pi}(s,a) \big|$.
\end{lemma} 

\begin{proof}
	Let us start rewriting the expression of the relative advantage $A_{P,\pi}^{P',\pi'}$:
    \begin{align*}
     A_{P,\pi}^{P',\pi'}(s) & =  \int_{\mathcal{A}} \pi'(a|s) \bigg( R(s,a) + \gamma \int_{\mathcal{S}} P'(s'|s,a) V^{P,\pi}(s') \mathrm{d}s'\bigg) \mathrm{d}a - V^{P,\pi}(s) = \\
     & = \int_{\mathcal{A}} \pi'(a|s) \bigg( R(s,a) + \gamma \int_{\mathcal{S}} P'(s'|s,a) V^{P,\pi}(s') \mathrm{d}s'\bigg) \mathrm{d}a - \int_{\mathcal{A}} \pi(a|s) \bigg( R(s,a) + \gamma \int_{\mathcal{S}} P(s'|s,a) V^{P,\pi}(s') \mathrm{d}s'\bigg) \mathrm{d}a = \\
     & = \int_{\mathcal{A}} \Big( \pi'(a|s) - \pi(a|s) \Big) R(s,a) \mathrm{d}a + \gamma \int_{\mathcal{A}}  \int_{\mathcal{S}} \Big( \pi'(a|s)P'(s'|s,a) - \pi(a|s) P(s'|s,a) \Big) V^{P,\pi}(s') \mathrm{d}s \mathrm{d}a = \\
     & = \int_{\mathcal{A}} \Big( \pi'(a|s) - \pi(a|s) \Big) R(s,a) \mathrm{d}a + \gamma \int_{\mathcal{A}}  \int_{\mathcal{S}} \Big( \pi'(a|s) - \pi(a|s) \Big) P(s'|s,a) V^{P,\pi}(s') \mathrm{d}s \mathrm{d}a + \\
     & \quad + \gamma \int_{\mathcal{A}}  \int_{\mathcal{S}} \pi'(a|s) \Big( P'(s'|s,a) - P(s'|s,a) \Big) V^{P,\pi}(s') \mathrm{d}s \mathrm{d}a =  \\
     & =  \int_{\mathcal{A}} \Big( \pi'(a|s) - \pi(a|s) \Big) \Big( R(s,a) + \gamma \int_{\mathcal{S}} P(s'|s,a) V^{P,\pi}(s') \mathrm{d}s\Big) \mathrm{d}a + \\
     & \quad + \gamma \int_{\mathcal{A}}  \int_{\mathcal{S}} \pi'(a|s) \Big( P'(s'|s,a) - P(s'|s,a) \Big) V^{P,\pi}(s') \mathrm{d}s\mathrm{d}a= \\
     & = \int_{\mathcal{A}} \Big( \pi'(a|s) - \pi(a|s) \Big) Q^{P,\pi}(s,a) \mathrm{d}a + \gamma \int_{\mathcal{A}}  \int_{\mathcal{S}} \pi'(a|s) \Big( P'(s'|s,a) - P(s'|s,a) \Big) V^{P,\pi}(s') \mathrm{d}s \mathrm{d}a,
    \end{align*}
    then we can obtain the proof straightforwardly, noticing that $\sup_{s\in\mathcal{S}} \Delta Q^{P,\pi}(s,\cdot) \le \Delta Q^{P,\pi}$ and recalling that $\Delta V^{P,\pi}\le \Delta Q^{P,\pi}$:
    \begin{align*}
    	\frac{\Delta A_{P,\pi}^{P',\pi'}}{2} \le \Big\|A_{P,\pi}^{P',\pi'} \Big\|_{\infty} \le D^{\pi',\pi}_{\infty} \frac{\sup_{s\in\mathcal{S}} \Delta Q^{P,\pi}(s,\cdot)}{2} + \gamma D^{P',P}_{\infty} \frac{\Delta V^{P,\pi}}{2} \le (D_{\infty}^{\pi', \pi} + \gamma D_{\infty}^{P', P}) \frac{\Delta Q^{P,\pi}}{2}.
    \end{align*}
\end{proof}

\boundUncoupled*
\begin{proof}
	Exploiting the bounds on the coupled expected relative advantage $\mathds{A}_{{P,\pi},\mu}^{P',\pi'}$ (Lemma~\ref{thr:boundAdvantage}), on the $\gamma$-discounted state distributions difference (Corollary~\ref{thr:distributionsBound}) and on $\Delta A_{P,\pi}^{P',\pi}$ (Lemma~\ref{thr:boundDeltaAdvantage}), we can state the following:	
	\begin{align*}
		J^{P',\pi'}_{\mu} - J^{P,\pi}_{\mu} 	 & \geq   \frac{\mathds{A}_{{P,\pi},\mu}^{P',\pi'}}{1 - \gamma} -  \frac{\big\| d^{P',\pi'}_{\mu} - d^{P,\pi}_{\mu} \big\|_{1}}{1 - \gamma}  \frac{\Delta A_{P,\pi}^{P',\pi'}}{2} \ge \\
		& \geq \frac{\mathds{A}_{{P,\pi},\mu}^{P,\pi'} + \mathds{A}_{{P,\pi},\mu}^{P',\pi}}{1 - \gamma} - \frac{\gamma}{(1 - \gamma)^{2}}  (D_{\ev}^{\pi', \pi} + D_{\ev}^{P', P}) (D_{\infty}^{\pi', \pi} + \gamma D_{\infty}^{P', P})  \frac{\Delta Q^{P,\pi}}{2} - \frac{\gamma}{1-\gamma}  D_{\ev}^{\pi', \pi} D_{\infty}^{P', P}  \frac{\Delta Q^{P,\pi}}{2} \ge \\
		& \geq \frac{\mathds{A}_{{P,\pi},\mu}^{P,\pi'} + \mathds{A}_{{P,\pi},\mu}^{P',\pi}}{1 - \gamma} - \frac{\gamma}{(1 - \gamma)^{2}}  \big( D_{\ev}^{\pi', \pi} D_{\infty}^{\pi', \pi}  +  D_{\ev}^{\pi', \pi} D_{\infty}^{P', P}  +  D_{\infty}^{\pi', \pi} D_{\ev}^{P', P}  +  \gamma  D_{\infty}^{P', P} D_{\ev}^{P', P} \big)  \frac{\Delta Q^{P,\pi}}{2}. \\
	\end{align*}
	With a factorization of the last expression we get the result.
\end{proof}

\optimalUpdate*
\begin{proof}
	Let us write explicitly the update coefficients in the decoupled bound (\ref{thr:boundUncoupled}):
	\begin{equation*}
		J^{{P}',\pi'}_{\mu}- J^{{P},\pi}_{\mu} \geq \frac{\alpha \mathds{A}_{{P,\pi},\mu}^{P,\overline{\pi}} + \beta \mathds{A}_{{P,\pi},\mu}^{\overline{P},\pi}}{1 - \gamma} - \frac{\gamma}{(1 - \gamma)^{2}}  \big( \alpha^2 D_{\ev}^{\overline{\pi}, \pi} D_{\infty}^{\overline{\pi}, \pi}  +  \alpha \beta D_{\ev}^{\overline{\pi}, \pi} D_{\infty}^{\overline{P}, P}  +  \alpha \beta D_{\infty}^{\overline{\pi}, \pi} D_{\ev}^{\overline{P}, P}  +  \gamma  \beta^2 D_{\infty}^{\overline{P}, P} D_{\ev}^{\overline{P}, P} \big)  \frac{\Delta Q^{P,\pi}}{2},
	\end{equation*}
	we now take the derivatives w.r.t. $\alpha$ and $\beta$ to find the stationary points:
	\begin{align*}
		\frac{\partial B}{\partial \alpha} & =  \frac{\mathds{A}_{{P,\pi},\mu}^{P,\overline{\pi}}}{1 - \gamma} - \frac{\gamma}{(1 - \gamma)^2}  \big( 2\alpha D_{\ev}^{\overline{\pi}, \pi} D_{\infty}^{\overline{\pi}, \pi}  +  \beta D_{\ev}^{\overline{\pi}, \pi} D_{\infty}^{\overline{P}, P}  +  \beta D_{\infty}^{\overline{\pi}, \pi} D_{\ev}^{\overline{P}, P} \big)  \frac{\Delta Q^{P,\pi}}{2}, \\
		\frac{\partial B}{\partial \beta} & = \frac{\mathds{A}_{{P,\pi},\mu}^{\overline{P},\pi}}{1 - \gamma} - \frac{\gamma}{(1 - \gamma)^2}  \big( \alpha D_{\ev}^{\overline{\pi}, \pi} D_{\infty}^{\overline{P}, P}  +  \alpha D_{\infty}^{\overline{\pi}, \pi} D_{\ev}^{\overline{P}, P}  +  2 \gamma  \beta D_{\infty}^{\overline{P}, P} D_{\ev}^{\overline{P}, P} \big)  \frac{\Delta Q^{P,\pi}}{2}.
	\end{align*}
	When the target policy is different from the current one and, symmetrically, the target model is different from the current model the linear system of the derivatives admits a unique solution. We compute the second order derivative to discover the nature of such point:
	\begin{align*}
		\frac{\partial B^2}{\partial^2 \alpha} & = - \frac{\gamma}{(1 - \gamma)^2}   D_{\ev}^{\overline{\pi}, \pi} D_{\infty}^{\overline{\pi}, \pi}  \Delta Q^{P,\pi}, \\
		\frac{\partial B^2}{\partial \alpha \partial \beta} = \frac{\partial B^2}{\partial \beta \partial \alpha} & = - \frac{\gamma}{(1 - \gamma)^2}  \big( D_{\ev}^{\overline{\pi}, \pi} D_{\infty}^{\overline{P}, P}  +  D_{\infty}^{\overline{\pi}, \pi} D_{\ev}^{\overline{P}, P} \big)  \frac{\Delta Q^{P,\pi}}{2}, \\
		\frac{\partial B^2}{\partial^2 \beta} & = - \frac{\gamma^2}{(1 - \gamma)^2}  D_{\infty}^{\overline{P}, P} D_{\ev}^{\overline{P}, P}  \Delta Q^{P,\pi},
	\end{align*}
	from the second order derivatives we can write the Hessian matrix:
	\begin{equation*}
		\mathcal{H}(\alpha,\beta) = - \frac{\gamma \Delta Q^{P,\pi}}{(1 - \gamma)^2} 
			\begin{pmatrix}
    			2 D_{\ev}^{\overline{\pi}, \pi} D_{\infty}^{\overline{\pi}, \pi} & D_{\ev}^{\overline{\pi}, \pi} D_{\infty}^{\overline{P}, P}  +  D_{\infty}^{\overline{\pi}, \pi} D_{\ev}^{\overline{P}, P} \\
    			D_{\ev}^{\overline{\pi}, \pi} D_{\infty}^{\overline{P}, P}  +  D_{\infty}^{\overline{\pi}, \pi} D_{\ev}^{\overline{P}, P} &  2 \gamma D_{\infty}^{\overline{P}, P} D_{\ev}^{\overline{P}, P} \\
  			\end{pmatrix},
	\end{equation*}
	having trace and determinant:
	\begin{align*}
		\mathrm{tr}(\mathcal{H}(\alpha,\beta)) & = - \frac{\gamma \Delta Q^{P',\pi}}{(1 - \gamma)^2}  \Big(2 D_{\ev}^{\overline{\pi}, \pi} D_{\infty}^{\overline{\pi}, \pi} + 2 \gamma D_{\infty}^{\overline{P}, P} D_{\ev}^{\overline{P}, P} \Big) \le 0 \\
		\mathrm{det}(\mathcal{H}(\alpha,\beta)) & =  \frac{\gamma^2 \Delta {Q_{P,\pi}^{P',\pi}}^2}{(1 - \gamma)^4} \Big( (4 \gamma - 2) D_{\ev}^{\overline{\pi}, \pi} D_{\infty}^{\overline{\pi}, \pi} D_{\infty}^{\overline{P}, P} D_{\ev}^{\overline{P}, P} - (D_{\ev}^{\overline{\pi}, \pi} D_{\infty}^{\overline{P}, P} )^2 - (D_{\infty}^{\overline{\pi}, \pi} D_{\ev}^{\overline{P}, P})^2 \Big) \le \\
		& \le \frac{\gamma^2 \Delta {Q_{P,\pi}^{P',\pi}}^2}{(1 - \gamma)^4} \Big( 2 D_{\ev}^{\overline{\pi}, \pi} D_{\infty}^{\overline{\pi}, \pi} D_{\infty}^{\overline{P}, P} D_{\ev}^{\overline{P}, P} - (D_{\ev}^{\overline{\pi}, \pi} D_{\infty}^{\overline{P}, P} )^2 - (D_{\infty}^{\overline{\pi}, \pi} D_{\ev}^{\overline{P}, P})^2 \Big) \le \\
		& \le - \frac{\gamma^2 \Delta {Q_{P,\pi}^{P',\pi}}^2}{(1 - \gamma)^4} \Big( D_{\ev}^{\overline{\pi}, \pi} D_{\infty}^{\overline{P}, P} - D_{\infty}^{\overline{\pi}, \pi} D_{\ev}^{\overline{P}, P}\Big) ^2 \le 0.
	\end{align*}
	When $\overline{P} \neq P$ and $\overline{\pi} \neq \pi$ we observe that the Hessian matrix
	is indefinite since both the trace and the determinant are negative. This means that
	the unique stationary point is a saddle point which is uninteresting for optimization
	purposes. By the way, $B(\alpha,\beta)$ is a quadratic function, therefore it is continuous
	on the compact set $[0,1]^2$ and therefore, from Weierstrass theorem, it admits a global
	maximum (and minimum). Since such point is not a stationary point it must lie on the boundary of $[0,1]^2$.

	Then, by setting to zero the equations $\frac{\partial B}{\partial \alpha} \big|_{\beta=0}$, $\frac{\partial B}{\partial \alpha} \big|_{\beta=1}$, $\frac{\partial B}{\partial \beta} \big|_{\alpha=0}$, $\frac{\partial B}{\partial \beta} \big|_{\alpha=1}$ we can obtain the following optimal values (which are clipped to lie in the interval $[0,1]$):
	\begin{align*}
		\alpha^*_0 = & \frac{(1-\gamma)}{\gamma \Delta Q^{P,\pi}} \frac{\mathds{A}_{P,\pi,\mu}^{P,\overline{\pi}}}{D_{\infty}^{\overline{\pi},\pi}D_{\ev}^{\overline{\pi},\pi}} \\
		\alpha^*_1 = & \frac{(1-\gamma)}{\gamma \Delta Q^{P,\pi}} \frac{\mathds{A}_{P,\pi,\mu}^{P,\overline{\pi}}}{D_{\infty}^{\overline{\pi},\pi}D_{\ev}^{\overline{\pi},\pi}} - \frac{1}{2}\bigg( \frac{D_{\ev}^{\overline{P},P}}{D_{\ev}^{\overline{\pi},\pi}} + \frac{D_{\infty}^{\overline{P},P}}{D_{\infty}^{\overline{\pi},\pi}} \bigg) \\
		\beta^*_0 = & \frac{(1-\gamma)}{\gamma^2 \Delta Q^{P,\pi}} \frac{\mathds{A}_{P,\pi,\mu}^{\overline{P},\pi}}{D_{\infty}^{\overline{P},P}D_{\ev}^{\overline{P},P}} \\
		\beta^*_1 = & \frac{(1-\gamma)}{\gamma^2 \Delta Q^{P,\pi}} \frac{\mathds{A}_{P,\pi,\mu}^{\overline{P},\pi}}{D_{\infty}^{\overline{P},P}D_{\ev}^{\overline{P},P}} - \frac{1}{2\gamma}\bigg( \frac{D_{\ev}^{\overline{\pi},\pi}}{D_{\ev}^{\overline{P},P}} + \frac{D_{\infty}^{\overline{\pi},\pi}}{D_{\infty}^{\overline{P},P}} \bigg) 
	\end{align*}
	Instead, for $\gamma\in (0,1)$, the Hessian is singular when either the target policy or the target model are equal to
	the current one. Those cases can be treated separately and clearly yield maxima points. When $\overline{P}= P$ then we have $\alpha^* = \alpha^*_0$, when $\overline{\pi} = \pi$ we have $\beta^* = \beta^*_0$.
	
We report the values of the decoupled bound $B(\alpha,\beta)$ (Theorem~\ref{thr:distributionsBound}) in correspondence of the optimal coefficients:
\begin{align*}
	B(\alpha^*_0, 0) &= \frac{(1-\gamma){\mathds{A}_{P,\pi,\mu}^{P, \overline{\pi}}}^2}{2\gamma \Delta Q^{P,\pi} D_{\infty}^{\overline{\pi},\pi} D_{\ev}^{\overline{\pi},\pi}}, \\
	B(0, \beta_0^*) &= \frac{(1-\gamma){\mathds{A}_{P,\pi,\mu}^{\overline{P},\pi}}^2}{2\gamma^2 \Delta Q^{P,\pi} D_{\infty}^{\overline{P},P} D_{\ev}^{\overline{P},P}}, \\
	B(1, \beta_1^*) &= \mathds{A}_{P,\pi,\mu}^{P, \overline{\pi}} + \frac{(1-\gamma){\mathds{A}_{P,\pi,\mu}^{\overline{P},\pi}}^2}{2\gamma^2 \Delta Q^{P,\pi} D_{\infty}^{\overline{P},P} D_{\ev}^{\overline{P},P}} - \frac{\mathds{A}_{P,\pi,\mu}^{\overline{P},\pi}}{2 \gamma} \bigg( \frac{D_{\ev}^{\overline{\pi},\pi}}{D_{\infty}^{\overline{P},P}} + \frac{D_{\infty}^{\overline{\pi},\pi}}{D_{\ev}^{\overline{P},P}} \bigg) + \\
	 & \quad + \frac{\Delta Q^{P,\pi}}{2(1-\gamma)} \bigg( \frac{1}{2} D_{\ev}^{\overline{P},P} D_{\ev}^{\overline{\pi},\pi} D_1 + \frac{1}{2} D_{\infty}^{\overline{P},P} D_{\ev}^{\overline{\pi},\pi} D_1 - \frac{1}{4} D_{\infty}^{\overline{P},P} D_{\ev}^{\overline{P},P} D_1^2 - \gamma D_{\ev}^{\overline{\pi},\pi} D_{\infty}^{\overline{\pi},\pi}  \bigg),\\
	 B(\alpha_1^*, 1) &= \mathds{A}_{P,\pi,\mu}^{\overline{P}, {\pi}} + \frac{(1-\gamma){\mathds{A}_{P,\pi,\mu}^{{P},\overline{\pi}}}^2}{2\gamma \Delta Q^{P,\pi} D_{\infty}^{\overline{\pi},\pi} D_{\ev}^{\overline{\pi},\pi}} - \frac{\mathds{A}_{P,\pi,\mu}^{{P},\overline{\pi}}}{2} \bigg( \frac{D_{\ev}^{\overline{P},P}}{D_{\infty}^{\overline{\pi},\pi}} + \frac{D_{\infty}^{\overline{P},P}}{D_{\ev}^{\overline{\pi},\pi}} \bigg) + \\
	 & \quad + \frac{\gamma \Delta Q^{P,\pi}}{2(1-\gamma)} \bigg( \frac{1}{2} D_{\ev}^{\overline{\pi},\pi} D_{\ev}^{\overline{P},P} D_2 + \frac{1}{2} D_{\infty}^{\overline{\pi},\pi} D_{\ev}^{\overline{P},P} D_2 - \frac{1}{4} D_{\infty}^{\overline{\pi},\pi} D_{\ev}^{\overline{\pi},\pi} D_2^2 - \gamma D_{\ev}^{\overline{P},P} D_{\infty}^{\overline{P},P}  \bigg),\\
\end{align*}
where 
\begin{align*}
	D_1 &= \frac{D_{\ev}^{\overline{\pi},\pi}}{D_{\ev}^{\overline{P}, P}} + \frac{D_{\infty}^{\overline{\pi},\pi}}{D_{\infty}^{\overline{P}, P}}, \\
	D_2 &=  \frac{D_{\ev}^{\overline{P},P}}{D_{\ev}^{\overline{\pi}, \pi}} + \frac{D_{\infty}^{\overline{P},P}}{D_{\infty}^{\overline{\pi}, \pi}}. \\
\end{align*}
\end{proof}

\oppositeSignAdv*
\begin{proof}
	Let us rewrite the expected relative advantage by decomposing $P_{\mathbr{\omega}}$:
	\begin{align*}
		A_{P_{\mathbr{\omega}}}^{P_i}(s,a) = & \int_{\mathcal{S}} \Big( P_i(s'|s,a) - P_{\mathbr{\omega}}(s'|s,a) \Big) U^{P_{\mathbr{\omega}}}(s,a,s') \mathrm{d}s' = \\
		& =\int_{\mathcal{S}} \bigg( P_i(s'|s,a) - \sum_{j=1}^M \omega_j P_j(s'|s,a) \bigg) U^{P_{\mathbr{\omega}}}(s,a,s') \mathrm{d}s'.
	\end{align*}
	Now we take the weighted sum of the previous equation:
	\begin{align*}
		\sum_{i=1}^M  \omega_i  A_{P_{\mathbr{\omega}}}^{P_i}(s,a) & = \sum_{i=1}^M  \omega_i \int_{\mathcal{S}} \bigg( P_i(s'|s,a) - \sum_{j=1}^M \omega_j P_j(s'|s,a) \bigg) U^{P_{\mathbr{\omega}}}(s,a,s') \mathrm{d}s'= \\
		& = \int_{\mathcal{S}} \bigg( \sum_{i=1}^M \omega_i P_i(s'|s,a) - \sum_{j=1}^M \omega_j P_j(s'|s,a) \bigg) U^{P_{\mathbr{\omega}}}(s,a,s') \mathrm{d}s' = 0,
	\end{align*}
	where we just observed that $\sum_{i=1}^M \omega_i P_i(s'|s,a) - \sum_{j=1}^M \omega_j P_j(s'|s,a) = 0$. 
\end{proof}

\optimalAdv*
\begin{proof}
	We first prove that the expected relative advantage \wrt the vertex models is non-positive and then we extend it to all the models. By contradiction, suppose there exists a vertex model $P_i \in \bm{P}$ having positive expected relative advantage. Then, we can perform a step of model update with SPMI starting from $P_{\mathbr{\omega}^*}$ and getting the new model $P_{\overbar{\mathbr{\omega}}}$ with a performance improvement of:
	\begin{equation*}	
		J^{P_{\overbar{\mathbr{\omega}}}} - J^{P_{\mathbr{\omega}^*}} \ge \frac{(1-\gamma) {\mathds{A}_{P_{{\mathbr{\omega}}^*},\mu}^{P_i}}^2}{2 \gamma^2 \Delta Q^{P,\pi} D^{P_i, P_{{\mathbr{\omega}}^*}}_{\infty} D^{P_{P_i, {\mathbr{\omega}}^*}}_{\ev}} > 0,
	\end{equation*}
	which is impossible as $P_{\mathbr{\omega}^*}$ is the optimal model. Let us consider a generic model $P_{\mathbr{\omega}}$, its advantage decomposes linearly in the vertex models:
	\begin{equation*}
		\mathds{A}_{P_{{\mathbr{\omega}}^*}, \mu}^{P_{\mathbr{\omega}}} = \sum_{i=1}^{M} \omega_i \mathds{A}_{P_{{\mathbr{\omega}}^*}, \mu}^{P_i} \le 0.
	\end{equation*}
	Let us now consider the subset of vertex models having non-zero coefficient for the optimal model $\{P_i \in \bm{P} : \omega^*_i > 0 \}$. From Lemma~\ref{thr:oppositeSignAdv} we have:
	\begin{equation}
		\sum_{i=1}^M \omega_i^* \mathds{A}_{P_{\mathbr{\omega}^*}, \mu}^{P_i} = \sum_{i: \omega_i^* > 0} \omega_i^* \mathds{A}_{P_{\mathbr{\omega}^*}, \mu}^{P_i} = 0.
	\end{equation}
	Since $\mathds{A}_{P_{\mathbr{\omega}^*}, \mu}^{P_i} \le 0$ from the first part of the theorem, it must be that all $\mathds{A}_{P_{\mathbr{\omega}^*}, \mu}^{P_i} = 0$. As an immediate consequence, all transition models in $ \mathrm{co}\big( \{P_i \in \bm{P} : \omega^*_i > 0 \} \big)$ must have zero expected relative advantage, due to the linear decomposition of the advantage.
\end{proof}

\performanceGap*
\begin{proof}
	Using Theorem~\ref{thr:perfImprovement} and Lemma~\ref{thr:lemmaAdvantage} we can write:
	\begin{align*}
		J^{P_{\mathbr{\omega}^*}}_\mu - J^{P_{\overbar{{\mathbr{\omega}}}}}_\mu & = \frac{1}{1-\gamma} \int_{\mathcal{S}} d_{\mu}^{P_{\mathbr{\omega}^*}}(s) \int_{\mathcal{A}} \pi(a|s) A_{P_{\overbar{{\mathbr{\omega}}}}}^{P_{\mathbr{\omega}^*}}(s,a) \mathrm{d}s \mathrm{d}a \le\\
		& \le \frac{1}{1-\gamma} \int_{\mathcal{S}} \int_{\mathcal{A}} \delta_{\mu}^{P_{\mathbr{\omega}^*}}(s,a)  \mathrm{d}s \mathrm{d}a \sup_{s\in\mathcal{S}, a\in\mathcal{A}} A_{P_{\overbar{{\mathbr{\omega}}}}}^{P_{\mathbr{\omega}^*}}(s,a) \le \\
		& \le \frac{1}{1-\gamma} \sup_{s\in\mathcal{S}, a\in\mathcal{A}} A_{P_{\overbar{{\mathbr{\omega}}}}}^{P_{\mathbr{\omega}^*}}(s,a),
	\end{align*}
	Now we observe that the relative advantage decomposes linearly in the target models:
	\begin{align*}
		A_{P_{\overbar{{\mathbr{\omega}}}}}^{P_{\mathbr{\omega}^*}}(s,a) = \sum_{i=1}^M\omega_i^* A^{P_i}_{P_{\overbar{\mathbr{\omega}}}}(s,a) \le \max_{i=1,2,...,M} A^{P_i}_{P_{\overbar{\mathbr{\omega}}}}(s,a),
	\end{align*}	 
	from which the theorem follows.
\end{proof}

\PGradientTheorem*
\begin{proof}
	We just rephrase the proof of the Policy Gradient Theorem~\cite{sutton2000policy}. Let us compute the gradient of the Q-function:
	\begin{align}
		\nabla_{\mathbr{\omega}} Q^{P_{\mathbr{\omega}}} (s,a) & = \nabla_{\mathbr{\omega}}  \int_{\mathcal{S}} P_{\mathbr{\omega}}(s'|s,a) U^{P_{\mathbr{\omega}}}(s,a,s') \mathrm{d}s' = \notag\\
		& =  \int_{\mathcal{S}} \bigg( \nabla_{\mathbr{\omega}}  P_{\mathbr{\omega}}(s'|s,a) U^{P_{\mathbr{\omega}}}(s,a,s') +  P_{\mathbr{\omega}}(s'|s,a) \nabla_{\mathbr{\omega}}  U^{P_{\mathbr{\omega}}}(s,a,s') \bigg) \mathrm{d}s'=  \label{line:PGradientTheorem1}\\
		& = \int_{\mathcal{S}} \bigg( \nabla_{\mathbr{\omega}}  P_{\mathbr{\omega}}(s'|s,a) U^{P_{\mathbr{\omega}}}(s,a,s') +  P_{\mathbr{\omega}}(s'|s,a) \nabla_{\mathbr{\omega}}  \bigg( R(s,a) + \gamma \int_{\mathcal{A}} \pi(a'|s') Q^{P_{\mathbr{\omega}}}(s',a') \mathrm{d}a \bigg) \bigg) \mathrm{d}s' =  \label{line:PGradientTheorem2} \\
		& = \int_{\mathcal{S}}  \nabla_{\mathbr{\omega}}  P_{\mathbr{\omega}}(s'|s,a) U^{P_{\mathbr{\omega}}}(s,a,s') \mathrm{d}s' +  \gamma \int_\mathcal{S} P_{\mathbr{\omega}}(s'|s,a)  \int_{\mathcal{A}} \pi(a'|s')  \nabla_{\mathbr{\omega}}  Q^{P_{\mathbr{\omega}}}(s',a') \mathrm{d}a \mathrm{d}s',\label{line:PGradientTheorem3}
	\end{align}
	where~\eqref{line:PGradientTheorem2} follows from~\eqref{line:PGradientTheorem1} by expressing the U-function with the Bellman equation. After unfolding~\eqref{line:PGradientTheorem3} we get:
	\begin{align*}
		\nabla_{\mathbr{\omega}} Q^{P_{\mathbr{\omega}}} (s,a) = \int_{\mathcal{S}} \int_{\mathcal{A}} \delta_{(s,a)}^{P_{\mathbr{\omega}}} (s'',a'') \int_{\mathcal{S}}  \nabla_{\mathbr{\omega}}  P_{\mathbr{\omega}}(s'|s'',a'') U^{P_{\mathbr{\omega}}}(s'',a'',s') \mathrm{d}s'' \mathrm{d}a'' \mathrm{d}s',
	\end{align*}
	where $\delta_{(s,a)}^{P_{\mathbr{\omega}}} (s'',a'')$ is the $\gamma$-discounted state-action distribution when forcing the first state to be $s$ and the first action to be $a$. We obtain the gradient of the expected return by observing that $J^{P_{\mathbr{\omega}}}_\mu = \int_{\mathcal{S}} \int_{\mathcal{A}} \mu(s) \pi(a|s) Q^{P_{\mathbr{\omega}}} (s,a) \mathrm{d}s \mathrm{d}a$ and therefore:
	\begin{align*}
		\nabla_{\mathbr{\omega}} J^{P_{\mathbr{\omega}}}_\mu = \int_{\mathcal{S}} \int_{\mathcal{A}} \mu(s) \pi(a|s) \nabla_{\mathbr{\omega}} Q^{P_{\mathbr{\omega}}} (s,a) \mathrm{d}s \mathrm{d}a =  \int_{\mathcal{S}} \int_{\mathcal{A}} \delta_{\mu}^{P_{\mathbr{\omega}}} (s'',a'') \int_{\mathcal{S}}  \nabla_{\mathbr{\omega}}  P_{\mathbr{\omega}}(s'|s'',a'') U^{P_{\mathbr{\omega}}}(s'',a'',s') \mathrm{d}s'' \mathrm{d}a'' \mathrm{d}s',
	\end{align*}
	by observing that $\int_{\mathcal{S}} \int_{\mathcal{A}} \mu(s) \pi(a|s) \delta_{(s,a)}^{P_{\mathbr{\omega}}} (s'',a'') \mathrm{d}s \mathrm{d}a = \delta_{\mu}^{P_{\mathbr{\omega}}} (s'',a'')$.
\end{proof}

\gradientAdv*
\begin{proof}
	Exploiting theorem~\eqref{thr:PGradientTheorem} and the definition of $P'$ we can write the expression of the gradient:
	\begin{align*}
		\frac{\partial J^{P'}_\mu}{\partial \beta} &= \int_{\mathcal{S}} \int_{\mathcal{A}} \delta_{\mu}^{P'}(s,a) \int_{\mathcal{S}} \frac{\partial}{\partial \beta} P'(s'|s,a) U^{P'}(s,a,s') \mathrm{d}s' \mathrm{d}a \mathrm{d}s = \\
		& = \int_{\mathcal{S}} \int_{\mathcal{A}} \delta_{\mu}^{P'}(s,a) \int_{\mathcal{S}} \Big( \overline{P} (s'|s,a) - P(s'|s,a) \Big) U^{P'}(s,a,s') \mathrm{d}s' \mathrm{d}a \mathrm{d}s = \\
		& = \sum_{i=1}^{M} \eta_i \int_{\mathcal{S}} \int_{\mathcal{A}} \delta_{\mu}^{P'}(s,a) \int_{\mathcal{S}} \Big( P_i (s'|s,a) -P(s'|s,a) \Big) U^{P'}(s,a,s') \mathrm{d}s' \mathrm{d}a \mathrm{d}s.
	\end{align*}
	For $\beta=0$ we have that $P'=P$ therefore:
	\begin{align*}
		\frac{\partial J^{P'}_\mu}{\partial \beta} \bigg\rvert_{\beta = 0} = \sum_{i=1}^{M} \eta_i \int_{\mathcal{S}} \int_{\mathcal{A}} \delta_{\mu}^{P}(s,a) \int_{\mathcal{S}} \Big( P_i (s'|s,a) -P(s'|s,a) \Big) U^{P}(s,a,s') \mathrm{d}s' \mathrm{d}a \mathrm{d}s = \sum_{i=1}^{M} \eta_i  \mathds{A}_{P,\mu}^{P_i}.
	\end{align*}
	
\end{proof}

\section{Dissimilarity functions for policies and models}
\label{apx:diss}
Given a set $\mathcal{X}$ a \emph{premetric} (or prametric, quasi-distance or divergence)~\cite{deza2009encyclopedia} is a function $f:\mathcal{X}\times \mathcal{X} \rightarrow \mathds{R}$ that satisfies the axioms:
\begin{itemize}
	\item $f(x,y) \ge 0, \; \forall x,y\in\mathcal{X}$ (non-negativity);
	\item $f(x,x) = 0, \; \forall x \in \mathcal{X}$ (reflexivity).
\end{itemize}
If $f$ satisfies also:
\begin{itemize}
	\item $f(x,y) = 0 \iff x = y, \; \forall x,y \in \mathcal{X}$ (identity of indiscernables);
	\item $f(x,y) = f(y,x),\; \forall x,y \in \mathcal{X}$ (symmetry);
	\item $f(x,y) \le f(x,z) + f(z,y),\; \forall x,y,z\in\mathcal{X}$ (triangular inequality).
\end{itemize}
then $f$ is a \emph{metric}.

The bounds presented in the paper make use of the new dissimilarity functions between policies and models defined as:
\begin{equation*}
	D_{\ev}^{\pi', \pi} = \ev_{s\sim d_{\mu}^{P,\pi}} \big\| \pi'(\cdot|s) - \pi(\cdot|s) \big\|_1, \qquad D_{\ev}^{P', P} = \ev_{(s,a)\sim \delta_{\mu}^{P,\pi}} \big\| P'(\cdot|s,a) - P(\cdot|s,a) \big\|_1,
\end{equation*}
in addition to the well-known ones:
\begin{equation*}
	D^{\pi',\pi}_{\infty}= \sup_{s\in\mathcal{S}} \big\| \pi'(\cdot|s) - \pi(\cdot|s) \big\|_1, \qquad D^{P',P}_{\infty}= \sup_{s\in\mathcal{S},a\in\mathcal{A}} \big\| P'(\cdot|s,a) - P(\cdot|s,a) \big\|_1.
\end{equation*}
The latter are clearly metric on the space of policies and models respectively, while the former are just premetric since, in general, they are not symmetric (since the expected value is computed under the $\gamma$-discounted distribution of $\pi$ or $P$) and they do not satisfy the triangular inequality. However the following result hold:

\begin{prop}
	Let $(P,\pi)$ and $(P',\pi')$ two model-policy pairs. If $D_{\ev}^{{P'}, P} = D_{\ev}^{\pi',\pi} = 0$ then $J_\mu^{P',\pi'} = J_\mu^{P,\pi}$.
\end{prop}
\begin{proof}
	If $D_{\ev}^{{P'}, P} = D_{\ev}^{\pi',\pi} = 0$ , from Proposition~\ref{thr:distributionsBound} we have that $\big\| d_\mu^{P',\pi'} - d_\mu^{P,\pi} \big\|_1 =0$ and therefore $d_\mu^{P',\pi'}(s) = d_\mu^{P,\pi}(s)$ for all $s\in\mathcal{S}$ being $\big\| \cdot - \cdot \big\|_1$ a metric. Let us consider the difference in the $\gamma$-discounted state-action distributions:
	\begin{align*}
		\delta_\mu^{P',\pi'}(s,a) - \delta_\mu^{P,\pi}(s,a) = \pi'(a|s)d_{\mu}^{P',\pi'}(s) - \pi(a|s)d_{\mu}^{P,\pi}(s) = \big(\pi'(sa|s) - \pi(a|s) \big) d_{\mu}^{P,\pi}(s),
	\end{align*}
	and the $\ell^1$-norm becomes:
	\begin{align*}
		\big\| \delta_\mu^{P',\pi'} - \delta_\mu^{P,\pi} \big\|_1 & = \int_{\mathcal{S}} d_{\mu}^{P,\pi}(s) \int_{\mathcal{A}} \Big| \pi'(a|s) - \pi(a|s)  \Big| \mathrm{d}a \mathrm{d}s = \\
		& = \ev_{s\sim d_{\mu}^{P,\pi}} \big\| \pi'(\cdot|s) - \pi(\cdot|s) \big\|_1 = D_{\ev}^{\pi',\pi}=0.
	\end{align*}
	Therefore $\delta_\mu^{P',\pi'}(s,a) = \delta_\mu^{P,\pi}(s,a)$ for all $s\in\mathcal{S}$ and all $a\in\mathcal{A}$ being $\big\| \cdot - \cdot \big\|_1$ a metric. Thus:
	\begin{equation*}
		J_\mu^{P',\pi'} = \frac{1}{1-\gamma} \int_{\mathcal{S}} \int_{\mathcal{A}} \delta_\mu^{P',\pi'}(s,a) R(s,a) \mathrm{d}a \mathrm{d}s  = \frac{1}{1-\gamma} \int_{\mathcal{S}} \int_{\mathcal{A}} \delta_\mu^{P,\pi}(s,a) R(s,a) \mathrm{d}a \mathrm{d}s = J_\mu^{P,\pi}.
	\end{equation*}
\end{proof}

\section{Examples of Conf-MDPs}
In this section we report two examples of Conf-MDPs having some interesting behaviors. In all figures, the transition probabilities are reported on the edges and 
the reward is written below the state name.

\subsection{An example of Conf-MDP with local optima}
\label{apx:C1}
Let us consider the Conf-MDP represented in Figure~\ref{fig:zero_adv_mdp} where $\omega \in [0 , 1]$ is the parameter, $p\in [ 0 , 1 ]$ is a small fixed probability (say $0.1$) and $M$ is a large positive number.
\begin{figure}[t]
	\includegraphics[width=\textwidth]{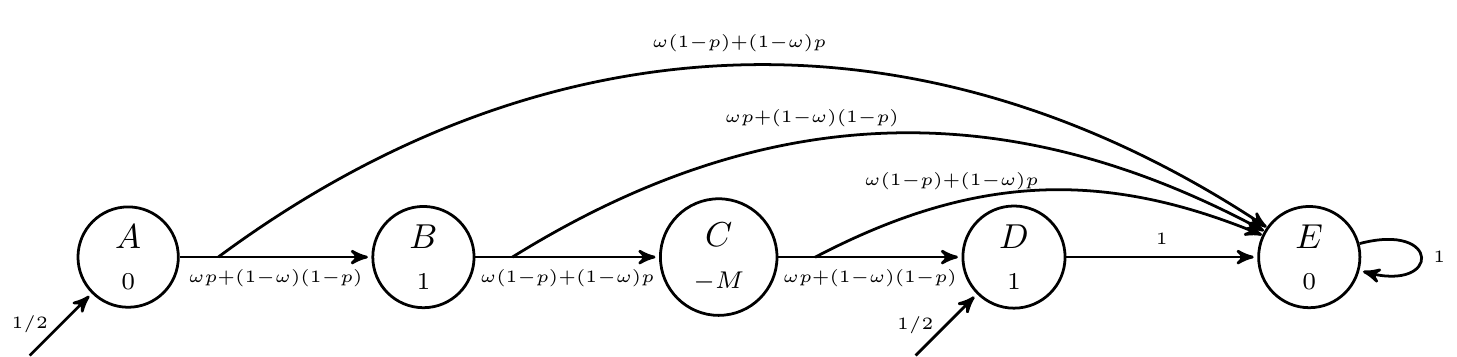}
	\caption{An example ofsome additional plots. Conf-MDP with local maxima.}\label{fig:zero_adv_mdp}
\end{figure} 
In each state there is only one action available (\ie all policies are optimal). The vertex models
are obtained for $\omega\in \{0,1\}$. For both target models there is a small probability to get the punishment $-M$ since for $\omega=0$ the probability to reach state $C$ from $B$ is $p$ and for $\omega=1$ state $B$ is reachable from $A$ with probability $p$. We expect that by mixing the two target models we can only worsen the performance. It is simple to realize that the expected return is a cubic function of $\omega$. We report the expression for $p=0.1$ and $\gamma=1$:
\begin{equation*}
	J^{P_{\omega}}_\mu = \frac{1}{2} \Big( 0.512 \omega^3 + (0.64M - 1.088) \omega^2 - (0.64M+0.296) \omega + 1.981 - 0.09 M \Big).
\end{equation*}
We can find the stationary points by looking at the derivative:
\begin{equation*}
	\frac{\partial J^{P_{\omega}}_\mu}{\partial\omega} = 0.768 \omega^2 + (0.64M - 1.088)\omega - 0.32M -0.148.
\end{equation*}
For $M$ sufficiently large the derivative has one sign variation thus it has two solutions of opposite sign, having expression:
\begin{equation*}
	\omega_{1,2} = \frac{1}{24} \Big( 17 - 10M \pm 10 \sqrt{M^2 - M -4} \Big).
\end{equation*}
Clearly, we are interested only in the solutions within $[0, 1]$ thus we discard the negative one. It is simple to see that the positive solution is approximately $\frac{1}{2}$ for $M$ sufficiently large, as:
\begin{equation*}
	\lim_{M\rightarrow +\infty}  \frac{1}{24} \Big( 17 - 10M + 10 \sqrt{M^2 - M -4} \Big) = \frac{1}{2}.
\end{equation*}
However, having a look at the second derivative we realize that this is a point of minimum, since
\begin{equation*}
	\frac{\partial^2 J^{P_{\omega}}_\mu}{\partial\omega^2} = 1.536 \omega + 0.64M  - 1.088 \big\rvert_{\omega = \frac{1}{2}} > 0.
\end{equation*}
Notice that in the unfortunate case in which SPMI is initialized at this value of $\omega$ the expected
relative advantage (which is the same as the gradient) is zero for both the vertex models and therefore there would be no update.
Therefore, the maximum must lie on the border, specifically either for $\omega=0$ or $\omega=1$. It is simple to see that $J^{P_1}_\mu > J^{P_0}_\mu$. Moreover, if we compute the value of the gradient for $\omega=0$ and $\omega=1$ we realize that in both cases the value is negative. Having a negative advantage, SPMI would never make any step even when the model is initialized at the lower performance vertex $\omega=0$.

\subsection{An example of Conf-MDP with mixed optimal model}
\label{apx:C2}
\begin{figure}[t]
	\includegraphics[width=\textwidth]{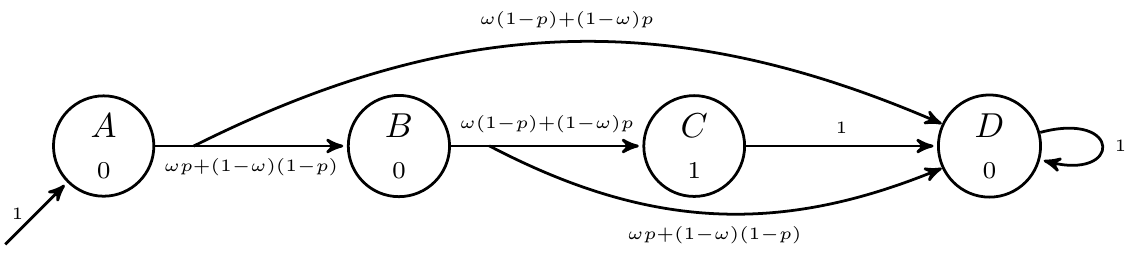}
	\caption{An example of Conf-MDP with mixed optimal model.}\label{fig:mdp_with_max}
\end{figure}
We consider the Conf-MDP as represented in Figure~\ref{fig:mdp_with_max}. As in the previous case, the parameter is $\omega\in [ 0, 1 ]$ and $p\in[0,1]$ is a fixed probability. We want to show that there exists no value of $\omega$ such that $P_{\omega}$ maximizes the value function in all states, while there exist one value of $\omega$ maximizing the expected return. It is
simple to compute the value function in each state:
\begin{flalign*}
	& V^{P_{\omega}}(A) = \gamma^2 \big( \omega p + (1-\omega)(1-p) \big) \big( \omega (1-p) + (1-\omega)p \big),\\
	& V^{P_{\omega}}(B) = \gamma \big( \omega (1-p) + (1-\omega)p \big),\\
	& V^{P_{\omega}}(C) = 1,\\
	& V^{P_{\omega}}(D) = 0. 
\end{flalign*}

Since the initial state is $A$ we have that $J^{P_{\omega}}_\mu = V^{P_{\omega}}(A)$ which is maximized for
$\omega=1/2$. However, there is no value of $\omega$ for which the value function of each state is maximized. As shown in Figure~\ref{fig:plot_value_mdp_with_max}, while $V^{P_{\omega}}(A)$ is maximal in $\omega=1/2$, $V^{P_{\omega}}(B)$ is maximal for $\omega = 1$. All values of $\omega \in [1/2 , 1]$ are indeed Pareto optimal (Figure~\ref{fig:plot_pareto_mdp_with_max}).

With some boring calculation we can determine the expression of the expected relative advantage functions:
\begin{align*}
	\mathds{A}_{P_{\omega},\mu}^{P_1} = \gamma^2 (1-\omega)(1-2\omega) (1-2p)^2  \\
	\mathds{A}_{P_{\omega},\mu}^{P_2} = -\gamma^2 \omega(1-2\omega) (1-2p)^2.
\end{align*}
We clearly see that they both vanish for $\omega=1/2$. 

\begin{figure}[h!]
\centering
\begin{minipage}{.45\textwidth}
  \centering
  \includegraphics[scale=1]{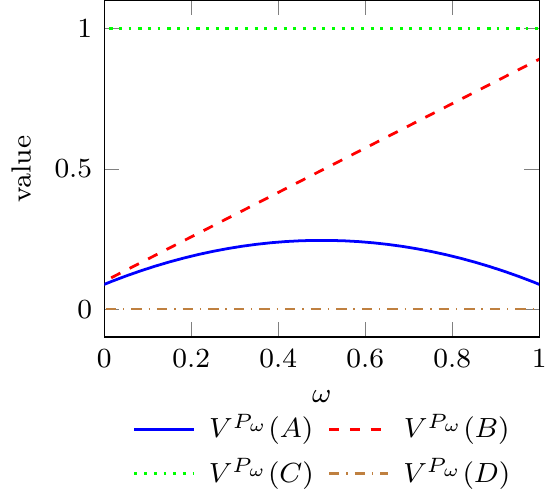}
  \captionof{figure}{The state value function of the Conf-MDP in Figure~\ref{fig:mdp_with_max} as a function of the parameter.}
  \label{fig:plot_value_mdp_with_max}
\end{minipage}%
\hfill
\begin{minipage}{.45\textwidth}
  \centering
  \includegraphics[scale=1]{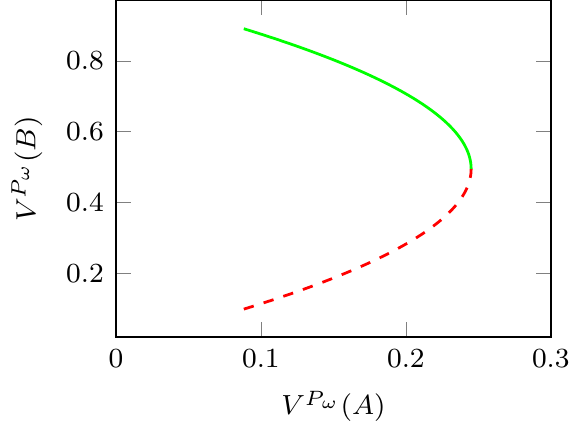}
  \captionof{figure}{The state value function of states $A$ and $B$ (the only ones varying with the parameter) of the Conf-MDP in Figure~\ref{fig:mdp_with_max}. The green continuous line is the Pareto frontier.}
  \label{fig:plot_pareto_mdp_with_max}
\end{minipage}
\end{figure}

\section{Environment description}
\label{apx:env_description}
In this appendix we provide a more detailed description of the environments used in the experimental evaluation.

\subsection{Student-Teacher domain}
\label{apx:student_description}
The teaching/learning process involves two entities: the teacher
and the student (learner) that interact. We assume both entities
share the same goal, i.e., maximizing the learning. The teaching
model, however, should be suited for the specific learning policy
of the student. For instance, not all students have the same skills and are able to capture the information provided by the teacher with the same speed and effectiveness. Thus, the teaching model should be tailored in order to meet the student's needs. Given the goal of maximizing learning, a teaching model induces an optimal learning policy (within the space of the policies that a certain student can play). Symmetrically, a learning policy determines an optimal teaching model (within the space of models available to the teacher). The question we want to answer in this experiment is: can we dynamically adapt the teaching model to the learning policy and the learning policy to the teaching model, so to maximize the learning?

We take inspiration from~\cite{rafferty2011faster} and we formalize the teaching/learning process as an MDP in which the student is the agent and the teacher is the environment. To fit our framework to this context, we can think to the teacher as an online learning platform that can be configured by the student in order to improve the learning experience. As in~\cite{rafferty2011faster} we test the model on the \quotes{alphabet arithmetic} a concept-learning task in which literals are mapped to numbers.

We consider $n$ literals $L_1, ..., L_n$, to which the student can assign the values $\{0,...,m\}$. The teacher, at each time step, provides an \quotes{example}, \ie an equation where a number of distinct literals (from $2$ to $p\le n$) sum to a numerical answer. The set of all possible examples
is given by:
\begin{equation*}
	\mathcal{E} = \Big\{ \sum_{i \in I} L_i = l \;: I \subseteq \{1,...,n\}, 2 \le |I| \le p , l \in \{0,...,|I|m\} \Big\}.
\end{equation*}
The student reacts to an example by performing an action, \ie an assignment of literals. The set of all assignments, \ie actions, is given by:
\begin{equation*}
	\mathcal{A} = \Big\{ L_1 = l_1, L_2 = l_2, \dots, L_n = l_n\;: l_i \in \{0,...,m\}, i=1,\dots,n \Big\},
\end{equation*}
thus $|\mathcal{A}|= (m+1)^n$. In order to model the student policy space we assume that a student can modify an arbitrary number of literals under the assumption that two consecutive assignments satisfy $\sum_{i=1}^n |l_i' - l_i| \le k$. This models the learning limitations of the student, in particular how hard is for the student to capture the teacher information. We assume that the teacher can provide any example. The set of states is the combination of an assignment and an example, \ie $\mathcal{S} = \mathcal{E} \times \mathcal{A}$.

The goal of the student is to perform assignments that are consistent with the teacher's examples (within its limitations on the possible assignments). So, while the student is learning the optimal policy it can configure the teacher to provide more suitable examples. The reward is $1$ when the assignment is consistent, $0$, when it is not and the horizon $H$ is finite. Notice that we don't have goal state, differently from~\cite{rafferty2011faster} so the problem becomes fully observable. We assume that, at the beginning, both policy and model are uniform distribution on the allowed actions/states.

\subsection{Racetrack Simulator}
\label{apx:racetrack_description}
The Racetrack Simulator aims to represent a basic abstraction of an autonomous car driving problem. 
In this context the autonomous driver, the learning agent, has to optimize a driving policy in order to run the vehicle to the track finish line as fast as possible. The vehicle and the track naturally compose the model of the learning process, however there is the possibility to tune a set of vehicle parameters, s.t. aerodynamic profile (to affect the vehicle stability) and engine setting.
Therefore, to maximize the performance, the driving policy of the agent and the model configuration has to be jointly considered. It is noteworthy that a specific model parametrization (vehicle setting) induces an optimal driving policy and, in the other hand, a driving policy determines an optimal model parametrization. Moreover a model-policy pair that results to be optimal for a specific track may not be optimal for a (morphologically) different track. Then, the question we aim to answer with this experiment is the following: can we learn the optimal model-policy pair for a given track by dynamically adapt the vehicle parametrization to the driving policy and, conversely, the driving policy to the vehicle parametrization during the learning process?

We formalize the learning process as an MDP in which the driver is the agent and the environment is composed by the track and the vehicle. The track is represented by a grid of positions, each grid point is either of type $roadway$, $wall$, $initial \; position$, $goal \; position$.
A state in the learning process belongs to the set:
\begin{equation*}\label{apx:student_exp}
	\mathcal{S} = \Big\{  (x, y, v_{x}, v_{y}) : x \in\{0,...,x_{\max}\}, y \in\{0,..., y_{\max}\}, v_{x} \in\{v_{\min},...,v_{\max}\}, v_{y}  \in\{v_{\min},...,v_{\max}\} \Big\},
\end{equation*}
where $(x, y)$ corresponds to a grid position and $(v_{x}, v_{y})$ are the speed along the coordinate axes.
At each step the agent can increment the speed along a coordinate direction or do nothing, then the action space is represented by the following:
\begin{equation*}
	\mathcal{A} = \Big\{ keep, increment \; v_{x}, increment \; v_{y}, decrement \; v_{x}, decrement \; v_{y}  \Big\}.
\end{equation*}
The learning process starts at the state corresponding to the initial position with zero velocities; the agent collects reward $1$ when it reaches a state corresponding to the goal position within the finite horizon $H$, he collects $0$ reward in any other case.

The transition model induces a success probability to any action, a failed action causes a random action to occur instead of the one selected by the agent. This probability aims to model the stability of the vehicle, the more the vehicle is unstable, the more is hard for the agent to drive it (or select an action). The model also induces a failure probability to every action: a failure represent a break of the vehicle, thus it directly cause the end of the episode. This feature represents the pressure on the vehicle engine, the more performance the driver asks for, the more it may break down.
We formalize the transition model as a convex combination between a set of vertex models: these correspond to vehicle configuration pushed towards the limit in terms of the aspects described above.
For our purpose we define a model dichotomy related to vehicle stability: $P \_ highspeed$ ($P\_hs$) trades stability at lower speed to have more stability (or high action success probability) in high speed situations, $P \_lowspeed$ ($P\_ls$), instead, provides more stability in low speed situation and poor stability at higher speed. We define also a model dichotomy related to engine boost: $P \_ boost$ ($P\_b$) guarantees higher engine performance and a lower reliability (or higher failure probability), at the opposite $P \_ noboost$ ($P\_nb$) provides higher reliability but poor engine performance. In Figure~\ref{fig:vertex_models} we propose a graphical representation of the features of these extreme models.

\begin{figure}[ht]
	\centering
	\includegraphics[scale=1]{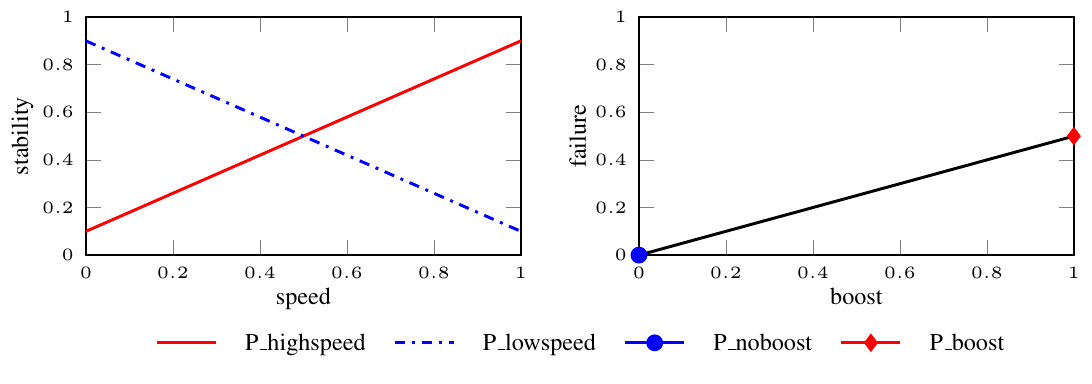}
	\caption{Graphical representation of the racetrack extreme models.}\label{fig:vertex_models}
\end{figure}

Considering any possible combination of stability and engine setting, we define the model set (set of vertex models) $\bm{P} = \{P\_hs\_b, P\_hs\_nb, P\_ls\_b , P\_ls\_nb\}$. Each model in this set is obtained by taking, for each state-action pair, the product of the transition probabilities of the components (\eg $P\_hs\_b(s,a) = P\_hs(s,a) \times P\_b(s,a)$). 
Then, we derive the model space as the convex hull of the vertices in the model set:
\begin{equation*}
	\mathcal{P}_{\omega} = \Big\{ P_{\omega} = \omega_0 P\_hs\_b + \omega_1 P\_hs\_nb + \omega_2 P\_ls\_b + \omega_3 P\_ls\_nb \Big\}.
\end{equation*}
While the agent is learning the optimal driving policy, the model parametrization can be configured (selecting a vector $(\omega_0, \omega_1, \omega_2, \omega_3)$) trying to fit the vehicle settings to the driving policy and simultaneously trying to fit the policy-settings pair to the morphology of the track. At the beginning of the learning process, we assume the policy to be a uniform distribution on the action space and the model to be $(0,0.5,0,0.5)$, that we can consider the most conservative parametrization in our context. We also report in Figure~\ref{fig:tracks} an illustrative representation of the tracks used in the experiments.

\begin{figure}[h]
	\centering
	\includegraphics[scale=1]{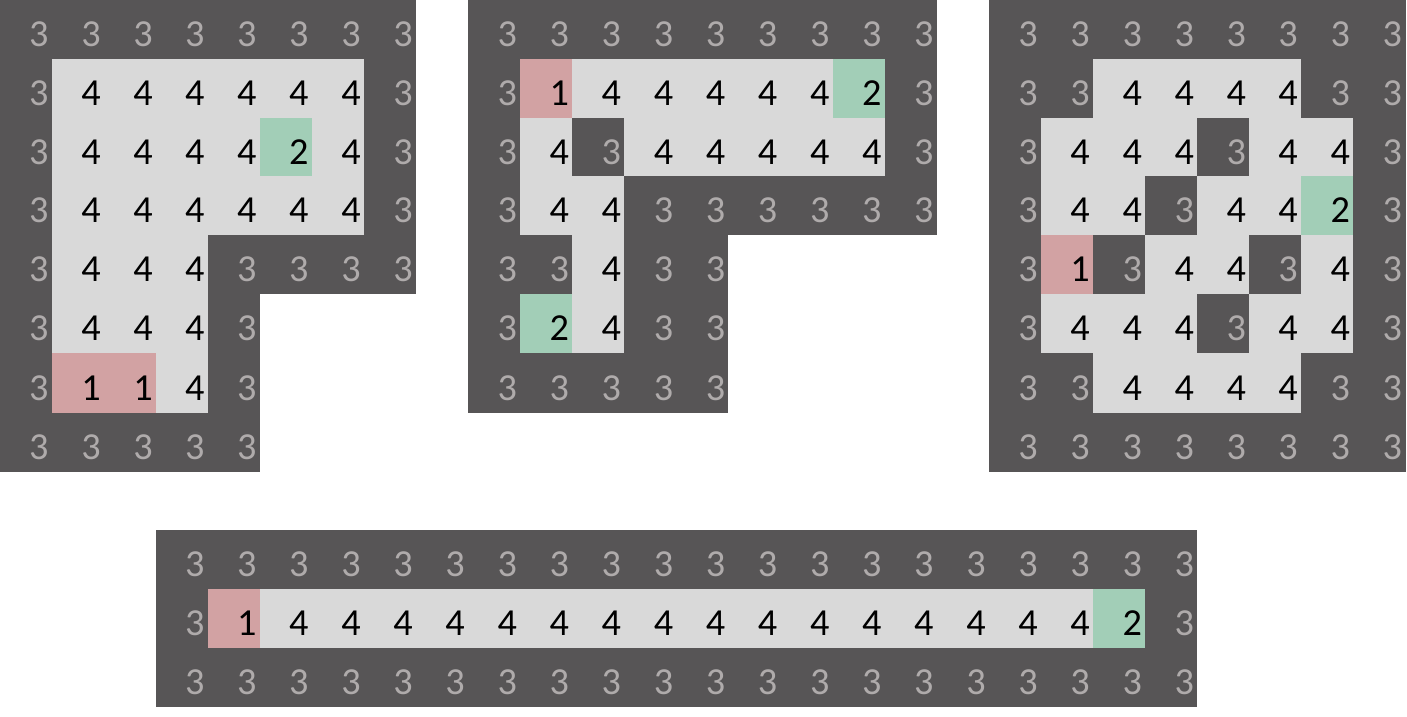}
	\caption{Basic representation of the tracks used in the Racetrack Simulator. From left to right: T1, T3, T4 and T2 just below. Each position have a type label: 1 for initial states, 2 for goal states, 3 for walls, 4 for roadtracks.}\label{fig:tracks}
\end{figure}

\section{Experiments details}
In this appendix, we report the hyper-parameters used for each experiment and some additional plots.

\subsection{Student-Teacher domain}
\label{apx:student_exp}
In Table~\ref{tab:stats_conv} we report the number of iterations to convergence for the different problem settings we considered. 
We can see that SPMI is the first or the second to converge in most of the cases. In Table~\ref{tab:hyper_student} we report the hyper-parameters we used for the runs in the different problem settings. In Figures~\ref{fig:2-1-1-2-10-all}, \ref{fig:2-1-1-2-10-notargettrick}, \ref{fig:2-3-1-2-all} we provide numerous plots showing different interesting metrics for the problem settings we considered in the experimental evaluation.

\begin{table}[h!]
	\caption{Number of steps to convergence for the update strategies in different
	problem settings of the Student-Teacher domain. In \textbf{bold} the best algorithm and \underline{underlined} the second best. The run has been stopped after 50000 iterations.}
	\label{tab:stats_conv}
	\vskip 0.15in
	\begin{center}
	\begin{small}
	\begin{tabular}{lccccc}
		\toprule
		Problem & SPMI & SPMI-sup & SPMI-alt & SPI+SMI & SPI+SMI \\
		\midrule
		2-1-1-2 & \underline{16234} & 18054 & 30923 & 22130 & \textbf{7705}\\
		2-1-2-2 &  \textbf{2839} & 3194 & 5678 &  \textbf {2839} & 12973\\
		2-2-1-2 & 20345 &  \underline{18287} & $>$50000 & 39722 & \textbf{10904}\\
		2-2-2-2 &  \textbf{12025} & \underline{14315} & $>$50000 & $>$50000 & 15257\\
		2-3-1-2 &  14187 & 13391 &  \textbf{11772} & $>$50000 & \underline{12183}\\
		3-1-1-2 &  \underline{15410} & {17929} & 22707 & 31122 & \textbf{14257}\\
		3-1-2-2 &  \textbf{3313} & 3313 & 8434 &  \textbf{3313} & 22846\\
		3-1-3-2 &  \textbf{2945} & 3435 & 5891 &  \textbf{2945} & 18090\\
		\bottomrule
	\end{tabular}
	\end{small}
	\end{center}
	\vskip -0.1in
\end{table}

\setlength{\tabcolsep}{0.2cm}
\begin{table}[h!]
	\caption{Additional information and hyper-parameters used for the different problem settings of the Student-Teacher domain.}
	\label{tab:hyper_student}
	\vskip 0.15in
	\begin{center}
	\begin{small}
	\begin{tabular}{*{7}{c}}
		\toprule
		Problem & $|\mathcal{S}|$ & $|\mathcal{A}|$ & horizon & $\Delta Q^{P,\pi}$ & $\gamma$ & $\epsilon$\\
		\midrule
		2-1-1-2 & 12 & 4 & 10 & $\frac{1-\gamma^{10}}{1-\gamma} \simeq\,$9.56 & 0.99 & 0\\
		2-1-2-2 & 12 & 4 & 10 & 9.56 & 0.99 & 0\\
		2-2-1-2 & 45 & 8 & 10 & 9.56 & 0.99 & 0\\
		2-2-2-2 & 45 & 8 & 10 & 9.56 & 0.99 & 0\\
		2-3-1-2 & 112 & 16 & 10 & 9.56 & 0.99 & 0\\
		3-1-1-2 & 72 & 9 & 10 & 9.56 & 0.99 & 0\\
		3-1-2-2 & 72 & 9 & 10 & 9.56 & 0.99 & 0\\
		3-1-3-2 & 72 & 9 & 10 & 9.56 & 0.99 & 0\\
		\bottomrule
	\end{tabular}
	\end{small}
	\end{center}
	\vskip -0.1in
\end{table}

\begin{figure}[h!]
	\includegraphics[scale=1]{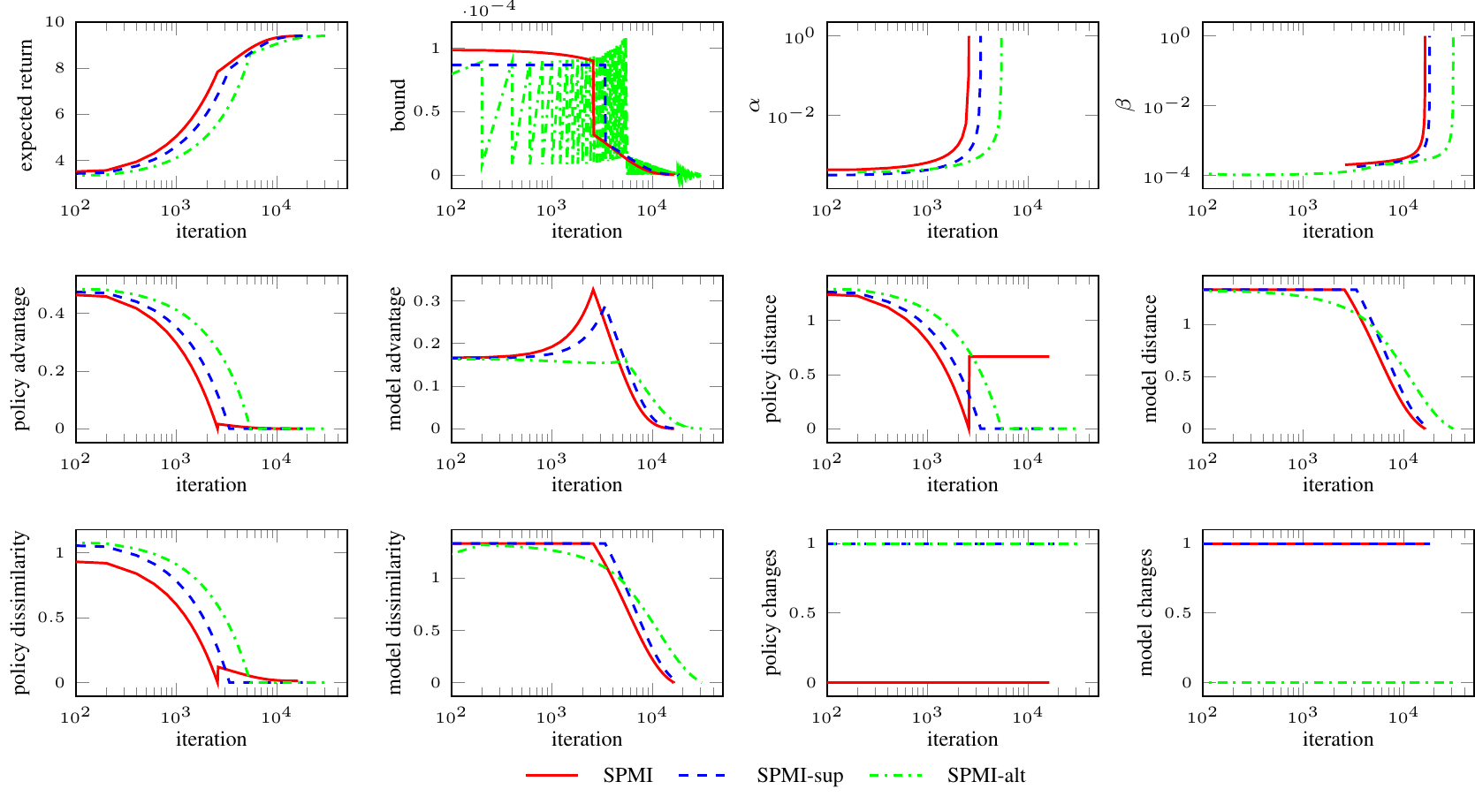}
	\caption{Several statistics of the persistent target choice for different update strategies in the case of Student-Teacher domain 2-1-1-2.}\label{fig:2-1-1-2-10-all}
\end{figure}
\begin{figure}[h!]
	\includegraphics[scale=1]{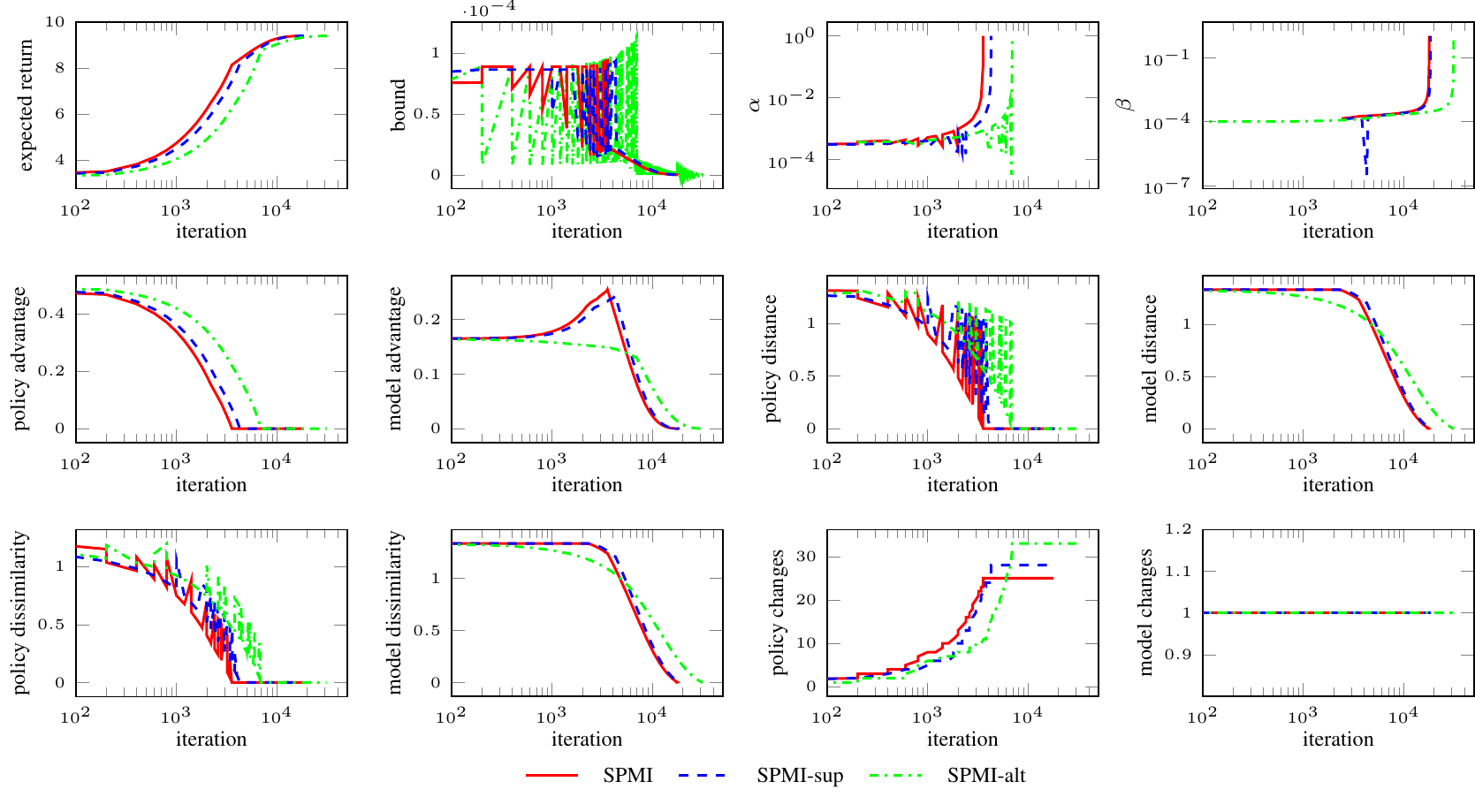}
	\caption{Several statistics of the greedy target choice for different update strategies in the case of Student-Teacher domain 2-1-1-2.}\label{fig:2-1-1-2-10-notargettrick}
\end{figure}
\begin{figure}[h!]
	\includegraphics[scale=1]{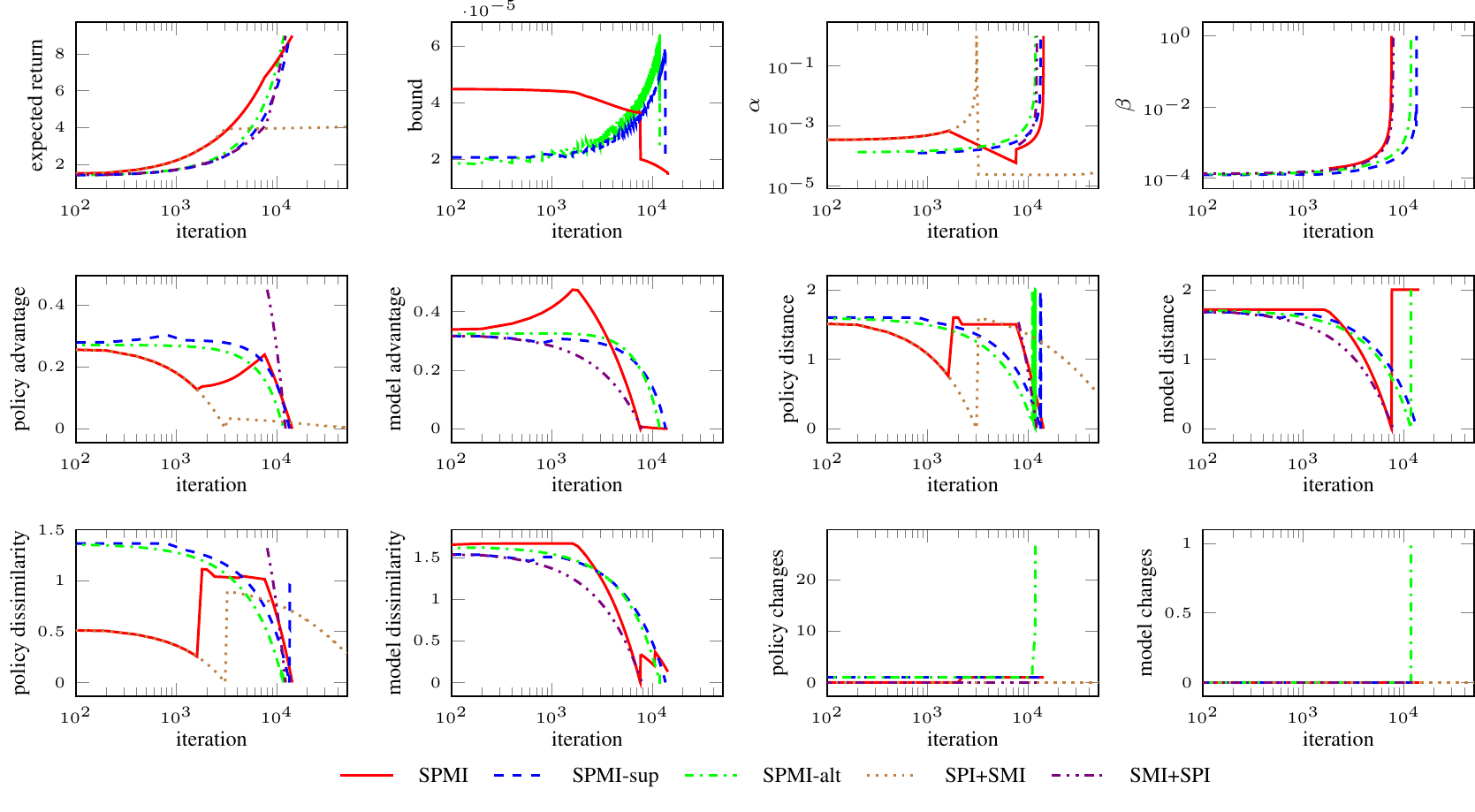}
	\caption{Several statistics of the persistent target choice for different update strategies in the case of Student-Teacher domain 2-3-1-2.}\label{fig:2-3-1-2-all}
\end{figure}
\clearpage

\subsection{Racetrack Simulator}
\label{apx:racetrack_exp}
In this section, we provide additional information about the experiments in the Racetrack Simulator environment (Table~\ref{tab:rcaetrack}) along with a more detailed presentation of the results reported in the Experimental Evaluation section of the paper (Figure~\ref{fig:twostart-2-all} and Figure~\ref{fig:straight-4-all}). Moreover, we present the performance obtained in some cases not mentioned in the paper (Figure~\ref{fig:2-all}).
\setlength{\tabcolsep}{0.2cm}
\begin{table}[ht]
	\caption{Additional information and hyper-parameters used for the different problem settings of the Racetrack Simulator. In the cases with two vertices only stability vehicle configurations are considered.}
	\label{tab:rcaetrack}
	\vskip 0.15in
	\begin{center}
	\begin{small}
	\begin{tabular}{*{8}{c}}
		\toprule
		Track & vertices & $|\mathcal{S}|$ & $|\mathcal{A}|$ & horizon & $\Delta Q^{P,\pi}$ & $\gamma$ & $\epsilon$\\
		\midrule
		 T1 & 2 & 675 & 5 & 20 & 1 & 0.9 & 0\\
		 T3 & 2 & 450 & 5 & 20 & 1 & 0.9 & 0\\
		 T4 & 2 & 675 & 5 & 20 & 1 & 0.9 & 0\\
		 T2 & 2 & 450 & 5 & 20 & 1 & 0.9 & 0\\
		 T2 & 4 & 451 & 5 & 20 & 1 & 0.9 & 0\\
		\bottomrule
	\end{tabular}
	\end{small}
	\end{center}
	\vskip -0.1in
\end{table}

\begin{figure}[h]
	\includegraphics[scale=1]{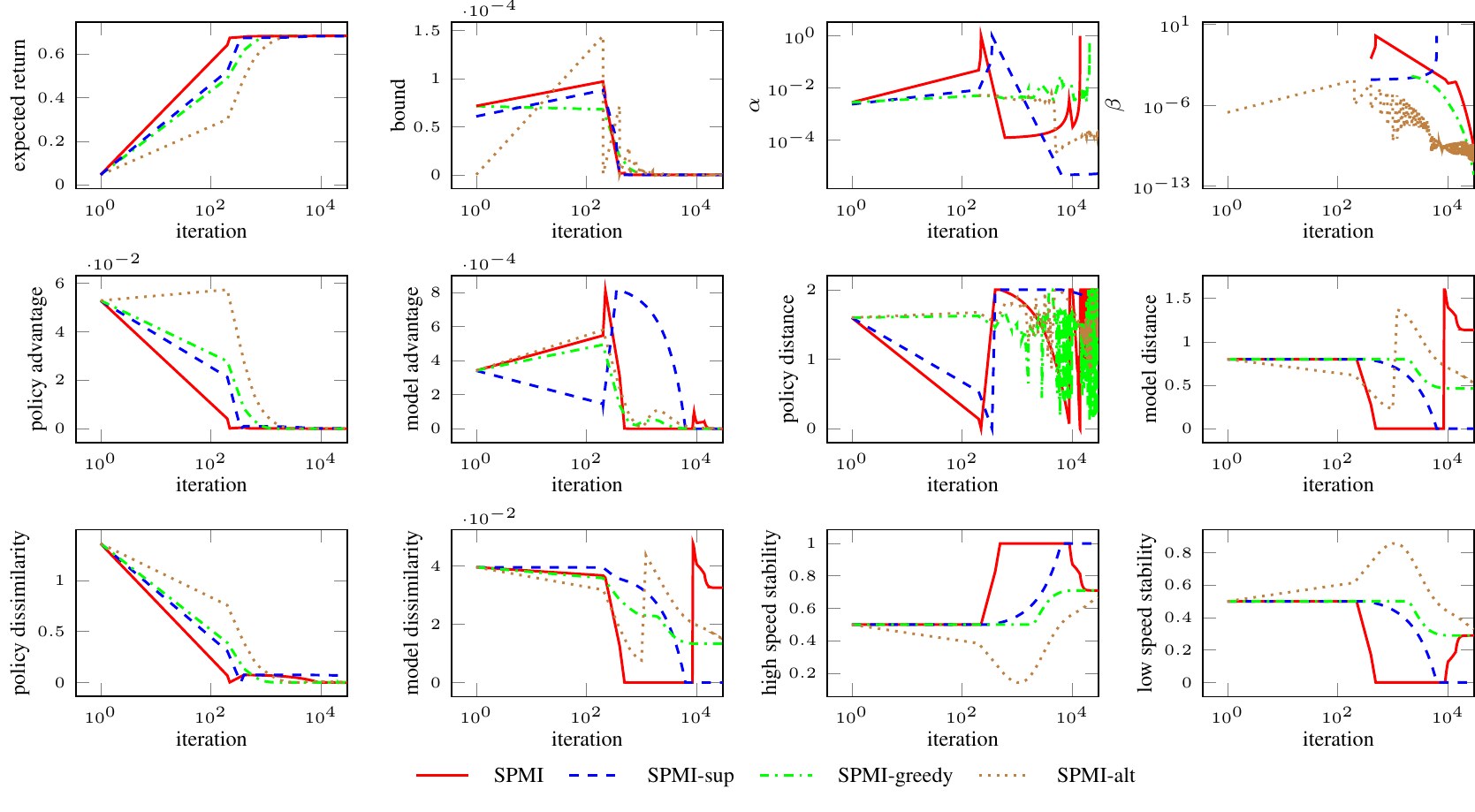}
	\caption{Several statistics of different update strategies and target choices in the case of Racetrack Simulator in the T1 considering vehicle stability only.}\label{fig:twostart-2-all}
\end{figure}
\begin{figure}[h]
	\includegraphics[scale=1]{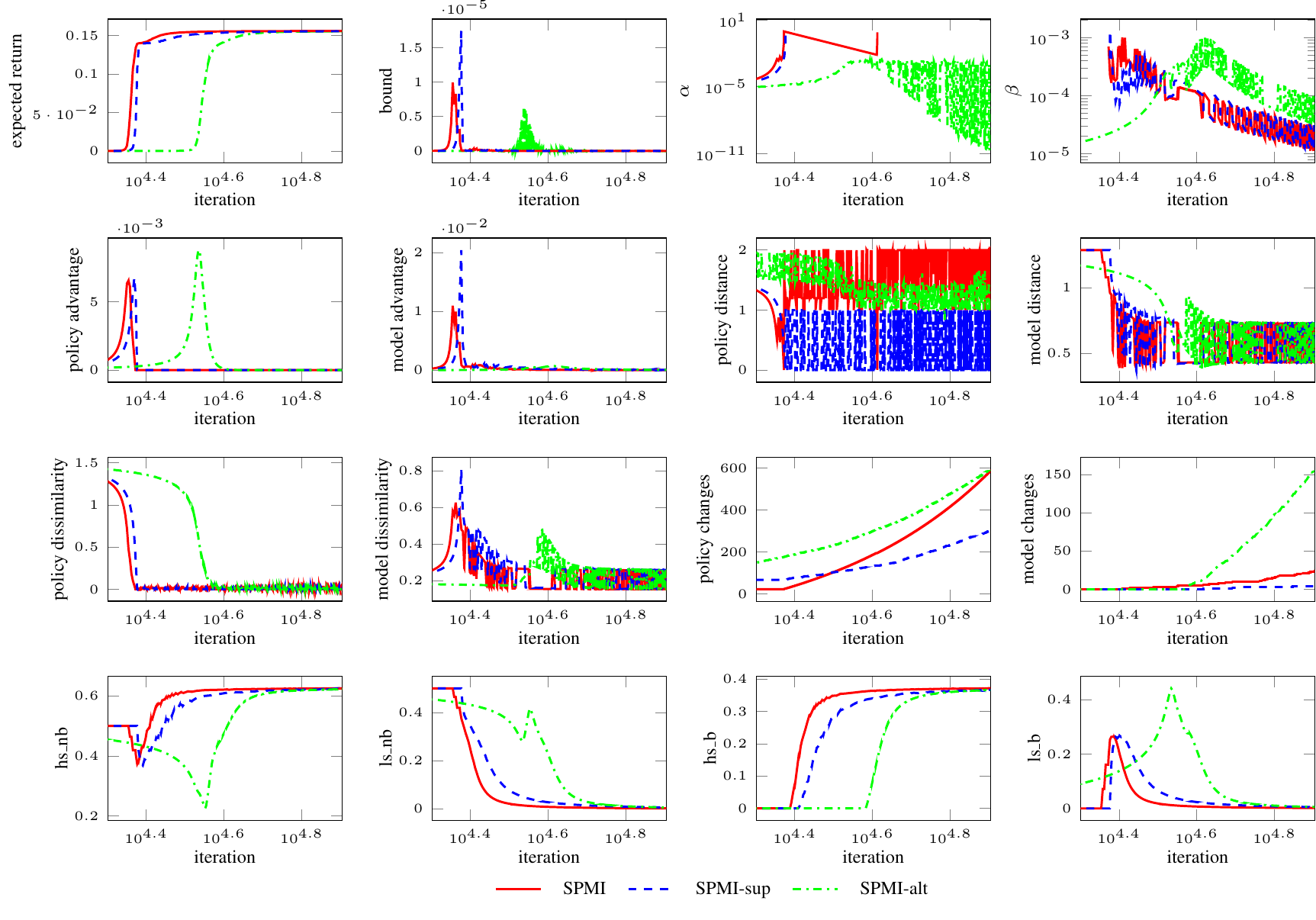}
	\caption{Several statistics of different update strategies in the case of Racetrack Simulator in the T2 considering vehicle stability and engine setting.}\label{fig:straight-4-all}
\end{figure}
\begin{figure}[h]
	\centering
	\includegraphics[scale=1]{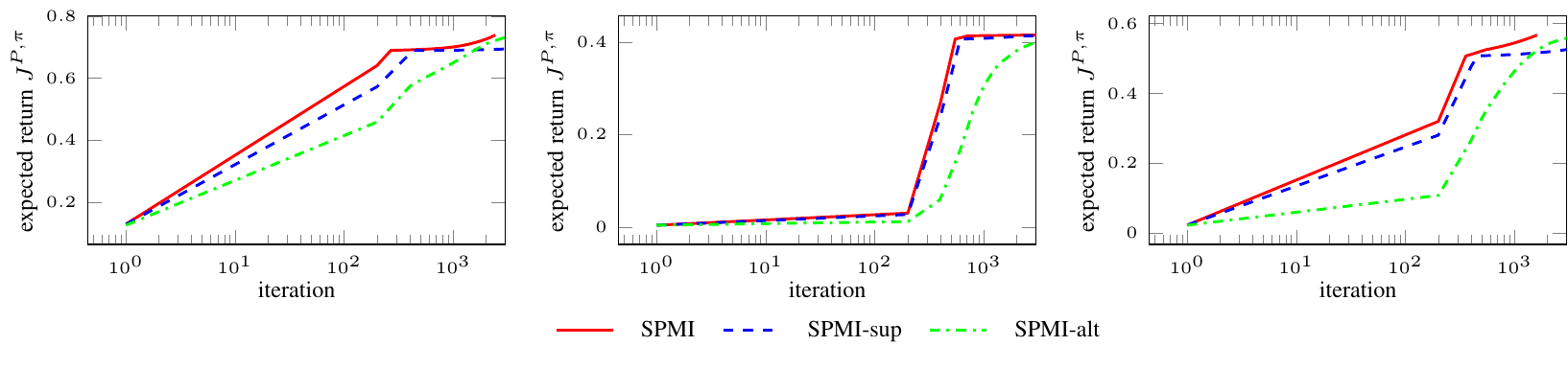}
	\caption{Expected return of the Racetrack Simulator in the T2, T3, T4 for different update strategies and considering vehicle stability configuration only.}\label{fig:2-all}
\end{figure}

\end{document}